\documentclass[sn-mathphys-num]{sn-jnl}

\usepackage{graphicx}%
\usepackage{multirow}%
\usepackage{amsmath,amssymb,amsfonts}%
\usepackage{amsthm}%
\usepackage{mathrsfs}%
\usepackage[title]{appendix}%
\usepackage{xcolor}%
\usepackage{textcomp}%
\usepackage{manyfoot}%
\usepackage{booktabs}%
\usepackage{algorithm}%
\usepackage{algorithmicx}%
\usepackage{algpseudocode}%
\usepackage{listings}%
\usepackage{cleveref}

\newtheorem{theorem}{Theorem}
\newtheorem{lemma}{Lemma}
\newtheorem{corollary}{Corollary}
\newtheorem{proposition}[theorem]{Proposition}%
\newtheorem{example}{Example}%
\newtheorem{remark}{Remark}%

\newtheorem{definition}{Definition}%

\usepackage{upgreek}
\usepackage{amsfonts, amsmath}
\usepackage{graphicx}
\usepackage{overpic}
\usepackage{epstopdf}
\usepackage{algorithm}
\usepackage{algpseudocode}
\usepackage{enumerate}
\usepackage{bm}
\ifpdf
  \DeclareGraphicsExtensions{.eps,.pdf,.png,.jpg}
\else
  \DeclareGraphicsExtensions{.eps}
\fi

\graphicspath{{image/}}

\makeatletter
\newcommand\xleftrightarrow[2][]{%
  \ext@arrow 9999{\longleftrightarrowfill@}{#1}{#2}}
\newcommand\longleftrightarrowfill@{%
  \arrowfill@\leftarrow\relbar\rightarrow}
\makeatother

\usepackage{verbatim} 

\usepackage{enumitem}
\setlist[enumerate]{leftmargin=.5in}
\setlist[itemize]{leftmargin=.5in}

\newtheorem{assumption}{Assumption}

\usepackage{etoolbox}
\AtEndEnvironment{remark}{\null\hfill$\Diamond$}
\AtEndEnvironment{assumption}{\null\hfill$\Diamond$}

\newcommand{\mcl}{\mathcal}

\newcommand{\mbf}{\mathbf}
\newcommand{\mbb}{\mathbb}

\newcommand{\supp}{{\rm supp\ }}

\newcommand{\dd}{{\rm d}}

\newcommand{\bx}{\mbf x}

\newcommand{\bz}{\mbf z}
\newcommand{\by}{\mbf y}

\newcommand{\bmm}{\mbf m}

\newcommand{\eps}{\epsilon}

\newcommand{\vphi}{\varphi}
\newcommand{\bphi}{{\bm{\phi}}}

\newcommand{\bzeta}{\bm{\zeta}}
\newcommand{\bvphi}{{\bm{\vphi}}}
\newcommand{\btheta}{\bm{\theta}}

\newcommand{\bchi}{\bm{\vphi}}

\newcommand{\mX}{\mcl{X}}
\newcommand{\mY}{\mcl{Y}}
\newcommand{\mH}{\mcl{H}}
\newcommand{\mK}{\mcl{K}}

\newcommand{\mB}{\mcl{B}}

\newcommand{\PP}{\mbb{P}}
\newcommand{\PC}{\mathcal{P}}
\newcommand{\R}{\mbb{R}}
\newcommand{\EE}{\mbb{E}}

\newcommand{\NN}{\mbb{N}}

\newcommand{\Law}{{\rm Law}}

\newcommand{\st}{{\rm\,s.t.}}

\definecolor{darkred}{rgb}{.7,0,0}

\definecolor{darkgreen}{rgb}{.15,.55,0}

\definecolor{darkblue}{rgb}{0,0,0.7}

\usepackage{amsopn}

\DeclareMathOperator*{\argmin}{arg\,min}
\DeclareMathOperator*{\argmax}{arg\,max}
\DeclareMathOperator*{\minimize}{{\rm minimize}}

  \AtEndEnvironment{example}{\null\hfill$\Diamond$}
  \AtEndEnvironment{runningexample}{\null\hfill$\Diamond$}

\begin{document}

\title[GPs conditioned on nonlinear PDEs]{Gaussian Measures Conditioned on Nonlinear Observations:
  Consistency, MAP Estimators, and Simulation}



\author[1]{\fnm{Yifan} \sur{Chen}}\email{yifan.chen@nyu.edu}

\author*[2]{\fnm{Bamdad} \sur{Hosseini}}\email{bamdadh@uw.edu}

\author[3]{\fnm{Houman} \sur{Owhadi}}\email{owhadi@caltech.edu}
\author[3]{\fnm{Andrew M} \sur{Stuart}}\email{stuart@caltech.edu}

\affil[1]{\orgdiv{Courant Institute of Mathematical Science}, \orgname{New York University}, \orgaddress{\city{New York},  \state{NY}, \country{USA}}}

\affil*[2]{\orgdiv{Department of Applied Mathematics}, \orgname{University of Washington}, \orgaddress{\city{Seattle},  \state{WA}, \country{USA}}}

\affil[3]{\orgdiv{Department of Computing and Mathematical Sciences}, \orgname{California Instutute of 
Technology}, \orgaddress{\city{Pasadena},  \state{CA}, \country{USA}}}


\abstract{
The article presents a systematic study of the problem of conditioning  a Gaussian random variable $\xi$
 on  nonlinear observations of the form $F \circ \bphi(\xi)$ where $\bphi: \mX \to \R^N$ is a bounded linear operator and $F$ is nonlinear. Such problems  arise in the context of Bayesian inference and recent machine learning-inspired PDE solvers.  We give a representer theorem for the conditioned random variable $\xi \mid F\circ \bphi(\xi)$, stating that it decomposes as the sum of an infinite-dimensional Gaussian (which is identified analytically) as well as a finite-dimensional non-Gaussian measure. We also introduce a novel notion of the mode of a conditional measure by taking the limit of the natural relaxation of the problem, to which we can apply the existing notion of  maximum a posteriori estimators of posterior measures. Finally, we introduce a variant of the  Laplace approximation for the efficient simulation of the aforementioned conditioned Gaussian random variables towards uncertainty quantification.
}




\maketitle

\section{Introduction}\label{sec:introduction}

We consider the problem of conditioning a Gaussian measure on a finite set of nonlinear observations
in the form of a nonlinear transformation of bounded linear functionals. 
Let $\{\mX, \langle \cdot, \cdot \rangle_\mX,  \| \cdot \|_\mX\}$ be a
separable Hilbert space with dual $\mX^\ast$ and consider a Gaussian measure 
$\mu = N(0, \mK) \in \PC(\mX)$, where $\PC(\mX)$ denotes the set of all Borel 
probability measures on $\mX.$ 
Let $\mK: \mX \to \mX$ denote the covariance operator under $\mu$. Fix a vector  $\bphi = (\phi_1, \dots, \phi_N) \in (\mX^\ast)^{\otimes N}$ for $N \in \NN$, 
along with a nonlinear map $F: \R^N \to \R^M$ for $M \in \NN.$
Let $\xi \sim \mu$ and $\beta>0$ be a parameter; then our goal in this 
article is to characterize the family of measures 
\begin{equation}
  \label{main-problem}
  \mu^\by_\beta := \Law\{   \xi \mid \by \sim  N( F( \bphi( \xi) ), 
\beta^{2} I) \},
\end{equation}
and their modes, in the limit of small $\beta$. 
The natural candidate for the $\beta=0$ limit is
\begin{equation}
  \label{main-problem2}
  \mu^\by_0 := \Law\{   \xi \mid F( \bphi( \xi) )=\by \}.
\end{equation}
We refer to the measures $\mu^\by_\beta$ for $\beta>0$ as \emph{posteriors} 
and to their $\beta=0$ limit $\mu^\by_0$ as \emph{conditionals.}
The modes of the posterior measures, 
which are the {\it maximum a posteriori (MAP)}
\footnote{This uses a specific
  choice for the definition of MAP estimator in the infinite-dimensional
  setting as there are several in the literature; this is 
discussed in detail in what follows.}
estimators of the measures $\mu_\beta^\by$, 
are defined via the family of optimization problems
\footnote{We will provide details ensuring
that $\mK^{-\frac12}$ and $\mK^{1/2}\mX$ are well-defined.}
\begin{equation}
  \label{MAP-optimization}
  u^\by_\beta := \argmin_{u \in \mK^{1/2}\mX} \: \:  \|\mK^{-1/2} u\|^2_{\mX}
  + \frac{1}{\beta^2} | F(\bphi(u)) - \by |^2.
\end{equation}
The natural candidate for the $\beta=0$ limit of the mode is
\begin{equation}\label{MAP-optimization2}
    \begin{aligned}
      & u^\by_0 := \argmin_{u \in \mK^{1/2}\mX} \: \: \| \mK^{-1/2}u\|_{\mX} \quad \text{subject to }(\st) \quad     F\bigl(\bphi(u)\bigr) = \by.
    \end{aligned}
  \end{equation}

We make the following contributions to understanding the
posteriors, MAP estimators and their $\beta \to 0$ limits:

\begin{enumerate}
\item We establish the existence of appropriate limits
of $\mu^\by_\beta$ and $u^\by_\beta$ as
$\beta \to 0$, making precise the natural candidates for
$\mu^\by_0$ and $u^\by_0$ defined above, and characterizing
$u^\by_0$ as an approximate definition of the MAP estimator
of the conditional measure $\mu^\by_0.$
These relationships, and the
theorems making them explicit, are
summarized in \Cref{fig:commut-diag}.

\item
We show that for $\beta \ge 0$, the posterior measures $\mu^\by_\beta$
can be decomposed in the convolution of a conditional Gaussian measure and a 
non-Gaussian measure that is finite-dimensional; this result is given in \Cref{prop:posterior-decomposition}. This decomposition is analogous to 
representer theorems for  the MAP estimator $u^\by_\beta$, stating that the 
minimizers of \eqref{MAP-optimization} are effectively finite-dimensional; see \Cref{lem:MAP-decomposition}.


\item We introduce a technique for
  generating samples from the posteriors $\mu^\by_\beta$ by decomposing them
  into a finite-dimensional component, which is sampled by standard
  algorithms such as Markov chain Monte Carlo (MCMC) or variational
  inference, and an infinite-dimensional Gaussian component, which may be simulated
 exactly using analytical properties of Gaussian measures; see \Cref{sec:sampling}.
 In particular, we show that the non-Gaussian component is amenable to approximation 
 using a Laplace or Gauss-Newton-type approximation in settings where lots of observations 
 are available, leading to efficient numerical algorithms in applications such as PDE solvers.



\end{enumerate}

\begin{figure}[htp]
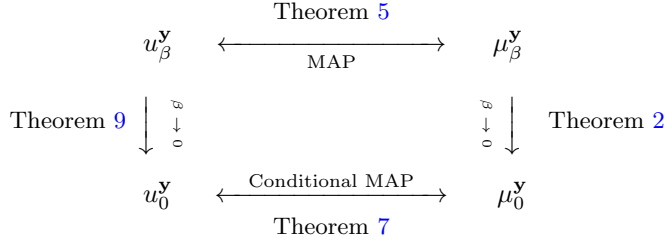

  \centering
  \begin{tabular}[htp]{c c c}
                    & {\scriptsize \Cref{General-MAP-characterization}} & \\
    \hspace{10ex} $u^\by_\beta$  & $\xleftrightarrow[\hspace{6.9ex}{\rm MAP}\hspace{6.9ex}]{}$  &
                                                                                     \hspace{-12ex}               $\mu^\by_\beta$
                                                                                    \vspace{2ex} \\
    \raisebox{2.5ex}{{\scriptsize \Cref{prop:map-convergence-to-conditional-mode}}}
    \rotatebox{90}{$\xleftarrow{ \hspace{4ex}}$} \raisebox{5ex}{ \rotatebox{-90}{\tiny $\beta \to 0$}}    & &
                                                                                                           \raisebox{5ex}{\rotatebox{-90}{\tiny $\beta \to 0$}}   \rotatebox{90}{$\xleftarrow{ \hspace{4ex}}$}   \hspace{1ex}
     \raisebox{2.5ex}{{\scriptsize \Cref{prop:convergence-of-conditionals-and-pushforwards}}}  
                                                                                                                      \vspace{.5ex} \\
    \hspace{10ex} $u^\by_0$  & $\xleftrightarrow{\hspace{2ex} {\text{Conditional  MAP}} \hspace{2ex}}$  &
                                                                                                              \hspace{-12ex}                $\mu^\by_0$ \\
    & {\scriptsize \Cref{prop:conditional-mode-optimization}} & \\
\end{tabular}
  \caption{Diagram relating small-noise limits of posteriors $\mu^\by_\beta$ and their MAP 
estimators $u^\by_\beta$ to their conditional counterparts $\mu^\by_0$ and $u^\by_0$.}
  \label{fig:commut-diag}
\end{figure}

\subsection{Motivating Examples}
Below, we give two motivating examples for the study of posterior measures of the form \eqref{main-problem} with their MAP estimators and conditional counterparts. 

\subsubsection{Inverse Problems}
\label{ssec:BIP}

Fix any $\beta>0$. Then the posterior measures $\mu^\by_\beta$ 
solve the Bayesian inverse problem (BIP) of finding the conditional 
distribution of $u\mid \by$ when
$u \sim \mu$, $\bzeta \sim N(0,\beta^2 I)$ independent of $u$,  and $\by$ (the data) is given 
by the model
\begin{subequations}
\label{eq:G}
\begin{align}
y&=G(u)+ \bzeta,\\
G&:=F\circ \bphi.
\end{align}
\end{subequations}
We can employ Bayes' rule \cite{stuart2010inverse}
to characterize the $\mu^\by_\beta$
via their Radon-Nikodym derivatives with respect to $\mu:$
\begin{equation}
  \label{main-problem-RN}
  \begin{aligned}
    \frac{\dd \mu^\by_\beta}{\dd \mu} (u)
    &= \frac{1}{\omega_\beta(\by)} \exp \left( - \frac{1}{2 \beta^2} | F( \bphi(u)) - \by |^2 \right), \\
    \omega_\beta(\by)
    &:= \EE_{u \sim \mu} \exp \left( - \frac{1}{2 \beta^2} | F( \bphi(u)) - \by |^2 \right).
\end{aligned}
\end{equation}
A common task in solving inverse problems and uncertainty quantification (UQ) is to 
estimate various statistics of the above posterior measures. The MAP $u^\by_\beta$ is a 
popular choice among practitioners, which highlights the importance of understanding 
its properties. Alternatively, one may choose to generate samples from $\mu^\by_\beta$ directly 
using MCMC and then compute empirical statistics such as posterior mean and variance. In either 
case, our finite-dimensional representations of $\mu^\by_\beta$ and its MAP $u^\by_\beta$ offer 
a path towards efficient calculations. Moreover,  it is natural to characterize solutions
of these problems in the small noise limit as $\beta \to 0$ to understand the consistency of 
the underlying inverse problems and their limit behavior.

\subsubsection{Solving PDEs with Gaussian Processes} \label{subsec:motivation}

One of the core problems of the field of scientific machine learning (ML)
is the design of novel algorithms for the solution of PDEs based on ML techniques.
An example of such a methodology was introduced by the authors in 
\cite{chen2021solving} where a Gaussian Process (GP) solver 
was developed for the numerical solution of nonlinear PDEs; henceforth referred to as GP-PDE. 
We briefly recall this
methodology in the context of a specific example from \cite{chen2021solving}. 
Consider the PDE
\begin{equation}
  \label{elliptic-PDE}
  \left\{
    \begin{aligned}
      - \Delta u (\bx) + \tau(u(\bx)) & = f(\bx), && \bx \in (0,1)^2,\\
         u(\bx) & = 0, && \bx \in \partial (0,1)^2,
    \end{aligned}
    \right.
  \end{equation}
  for $\tau:\R \to \R$ and $f: (0,1)^2 \to \R.$ We assume the existence of a unique solution $u^\star$ in the strong/classical sense. Then GP-PDE aims to find a numerical approximation 
$u_\beta^\by$ to $u^\star$ by the following recipe: First, choose a set of $M$
  collocation points $\bx_1, \dots, \bx_M \in [0,1]^2$, with $J$ in the interior and $M-J$ on the boundary, ordered so that $\bx_1, \dots, \bx_J \in (0,1)^2$
  while $\bx_{J+1}, \dots, \bx_{M} \in \partial (0,1)^2$. Then define
$\bphi = ( \phi_1, \dots, \phi_N)$ for $N = J + M$ with the $\phi_j$ defined as
  \begin{equation}
\label{PDE-MAP-optimization1}
    \begin{aligned}
      \phi_j(u) & = u(\bx_j),  && \text{ for } j = 1, \dots, M, \\
      \phi_j(u) & = \Delta u(\bx_{j-M}), && \text{ for } j = M+1, \dots, M+J.
  \end{aligned}
\end{equation}
and the nonlinear function  $F: \R^{M + J} \to \R^M$ defined row-wise as
\begin{equation}
\label{PDE-MAP-optimization2}
  F_j(\bz) := \left\{
    \begin{aligned}
      &  - z_{j + M} + \tau(z_j), &&  1 \le j \le J, \\
      &   z_j, && J+1 \le j \le M,
    \end{aligned}
    \right.
\end{equation}
Furthermore, define the vector $\by \in \R^M$ defined element-wise as
  \begin{equation}
\label{PDE-MAP-optimization3}
  y_j := \left\{
    \begin{aligned}
      &  f(\bx_j), &&  1 \le j \le J, \\
      &   0, && J+1 \le j \le M.
    \end{aligned}
    \right.
  \end{equation}
With these definitions, we may now consider the optimization
problems \eqref{MAP-optimization} and \eqref{MAP-optimization2}.
The two optimization problems define two variants of GP-PDE, one leading to a 
constrained optimization problem, and another being its unconstrained relaxation.

Recalling the discussion earlier in \Cref{sec:introduction} suggests that the resulting 
minimizer identifies the mode of an underlying posterior measure 
$\mu^\by_\beta$. This observation was discussed informally 
in \cite{chen2021solving}, in the setting of the GP-PDE methodology,
and our \Cref{General-MAP-characterization} establishes this connection rigorously. 
The GP-PDE methodology relies on a representer 
theorem (see also \cite{smola1998learning}) that identifies 
the solution of \eqref{MAP-optimization}, in the GP-PDE context, 
via a finite-dimensional optimization problem. In
\cite{chen2021solving} it is argued 
that the natural $\beta \to 0$ limit of \eqref{MAP-optimization}, 
namely \eqref{MAP-optimization2}, 
can also be solved with a representer theorem. 
\Cref{prop:posterior-decomposition} and  \Cref{lem:MAP-decomposition} can be viewed as establishing
Bayesian analogs of these results from \cite{chen2021solving}, 
where the exposition is primarily focused on kernel methods. 
  
\subsection{Literature Review} \label{subsec:lit}
Below we give an overview of the relevant literature to our work with a particular focus on 
the theory of Bayesian inverse problems, GPs, and probabilistic methods in numerical analysis.

\subsubsection{Bayesian Inverse Problems and MAP Estimators}
Bayesian inference \cite{gelman1995bayesian} is a cornerstone of modern statistics and 
data science. When applied in the context of infinite-dimensional or functional inference 
the methodology is best known under the term Bayesian inverse problems \cite{franklin1970well, tarantola2005inverse, kaipio2006statistical, stuart2010inverse}. Over the past decade, the algorithmic development and 
theoretical analysis for Bayesian inverse problems have become mature areas of research. Here, 
Bayesian inference with Gaussian prior measures is by far the most common setup for both algorithms 
and theoretical analysis. The overwhelming majority of 
function space MCMC algorithms \cite{tierney1998note, beskos2011hybrid, cotter2013mcmc, cui2016dimension, beskos2017geometric} are developed specifically for Gaussian priors; see
\cite{hosseini2019two} and references within for examples of algorithms for
non-Gaussian priors. 
The well-posedness theory of Bayesian inverse problems was 
originally developed for the case of Gaussian (or sub-Gaussian) priors \cite{cotter2009bayesian, stuart2010inverse} and was later extended to the non-Gaussian setting 
\cite{dashti2012besov, hosseini2017well, hosseini2017well-2, sullivan2017well, sprungk2020local, latz2020well} but the case of Gaussian priors remains most applicable as it allows for 
the widest range of nonlinear forward maps. From this perspective, this article 
makes important theoretical contributions towards the understanding and characterization of 
Bayesian posteriors under nonlinear observation models with Gaussian priors. Most importantly, 
our second contribution enables the use of finite-dimensional MCMC algorithms for nonlinear 
observation models without the need for direct discretization of the inverse problem. 

Variational methods are also an important family of algorithms for the solution of Bayesian inverse problems. 
Perhaps the most common task here is  computing a MAP estimator. Defining a MAP 
estimator in the function space setting is highly non-trivial. 
Several definitions, and resulting analyses, of modes of measures 
on  infinite-dimensional spaces exist \cite{agapiou2018sparsity,ayanbayev2021convergence,ayanbayev2021gamma,clason2019generalized,dashti2013map,helin2015maximum} where the choice of the notion of the mode is closely tied to the 
choice of the prior measure. Once again, the  Gaussian priors
lead to the most natural definition of a MAP estimator  
\cite{dashti2013map, ikeda2014stochastic} which is the same one we shall use to 
define $u^\by_\beta$ in \eqref{MAP-optimization}. However, to our knowledge, 
a notion of a conditional mode, i.e., a precise definition of $u^\by_0$ as in \eqref{MAP-optimization2}
has not been studied before and constitutes one of our main contributions.

\subsubsection{Gaussian Measures and Processes}
The general theory of Gaussian measures in infinite-dimensional settings 
is a classic subject in probability theory and the theory of stochastic differential equations. 
We refer the reader to the work of Bogachev  \cite{bogachev1998gaussian} for the detailed treatment 
of this subject on topological vector spaces and Maniglia and Rhandi \cite{maniglia2004gaussian} 
and Janson \cite{janson1997gaussian}
for the case of Hilbert spaces.

GPs, as a special instance of Gaussian measures, and, by extension, reproducing kernel Hilbert space (RKHS) methods \cite{kanagawa2018gaussian, van2008reproducing} and support 
vector machines \cite{smola1998learning}, have 
a long history in approximation theory \cite{wendland2004scattered}, 
statistical modeling and inference \cite{gine2021mathematical},
inverse problems \cite{cressie1990origins}, and
machine learning \cite{williams2007gaussian, smola1998learning}. 
While in this article we mainly focus on solving differential equations with GPs as an application 
of our theory \cite{sarkka2011linear, Owhadi:2014, chkrebtii2016bayesian, cockayne2017probabilistic, raissi2018numerical, swiler2020survey, chen2021solving, wang2021bayesian} (see also \Cref{sec:probabilistic-methods-for-solving-PDEs} below), GPs 
have wide applications in many modern areas of scientific computing and machine learning 
such as deep GPs \cite{damianou2013deep, dunlop2018deep, jakkala2021deep, dutordoir2021deep, owhadi2023ideas} as a model for deep learning, vector-valued GPs for operator learning 
\cite{batlle2024kernel} and  generative modeling \cite{murray2008gaussian, casale2018gaussian, fortuin2020gp, pandey2024diffeomorphic}, and graphical models for semi-supervised learning \cite{bertozzi2018uncertainty}.

The reasons for this widespread use of
GPs are their many desirable theoretical 
properties that lead to efficient algorithms. Perhaps the most useful are 
the facts that (1) GPs are completely identified by their mean and covariance operators; (2) 
GPs are closed under affine transformations; and (3) GPs conditioned 
on affine observations are also GPs that can be identified analytically; see \Cref{prop:conditional-Gaussian-direct-sum}. However, GPs conditioned on nonlinear observations are in general no longer 
GPs and cannot be identified analytically. Due to this fact, such conditional measures are 
often characterized computationally using MCMC \cite{robert1999monte, cotter2013mcmc, beskos2017geometric} or variational inference \cite{blei2017variational, pinski2015kullback}.
Such conditional measures are readily common in the field of inverse problems but they 
are increasingly common in modern machine learning applications mentioned in the previous 
paragraph as well. To this end, 
one of the main contributions of this article is to reveal the additional structure of conditioned 
GPs in the nonlinear setting that can be further leveraged by both MCMC and variational algorithms 
to further improve the accuracy and complexity of algorithms.



\subsubsection{The Intersection of Numerical Analysis and Probability}\label{sec:probabilistic-methods-for-solving-PDEs}
As discussed in \cite{owhadi2019statistical}, the fields of numerical approximation and statistical inference, traditionally viewed as distinct, are in fact deeply connected through their 
 common purpose of making estimations with partial information \cite[Chap.~20]{owhadi2019operator}.
 This shared purpose has recently stimulated a growing interest in   learning approaches to solving PDEs \cite{Owhadi:2014, raissi2017inferring}  and in   the merging of numerical errors with modeling errors and UQ \cite{hennig2015probabilistic}.
Although this trend may seem novel, the synergy between numerical approximation and statistical inference has historical roots, dating back to Poincaré's lectures on Probability Theory \cite{Poincare:1896}, and extending through the pioneering work of Sul'din \cite{sul1959wiener}, Palasti and Renyi \cite{PalastiRenyi1956}, Sard \cite{sard1963}, Kimeldorf and Wahba \cite{Kimeldorf70},   and Larkin \cite{larkin1972gaussian}. While these studies initially
``attracted little attention among numerical analysts'' \cite{larkin1972gaussian}, they were revived 
in the fields of Information Based
Complexity \cite{Traub1988}, Bayesian Numerical Analysis \cite{Diaconis:1988}, and more recently  in
 Probabilistic Numerics \cite{hennig2015probabilistic,cockayne2019bayesian}. This connection 
 between inference and numerical approximation is also central to Bayesian/decision-theoretic approaches to solving ODEs \cite{Skilling1992} and PDEs \cite{Owhadi:2014}, in identifying operator adapted wavelets \cite{owhadi2019operator} and designing fast solvers
 for kernel matrices \cite{schafer2017compression, schafer2020sparse,chen2023sparse}, and in parameter estimation \cite{chen2021consistency}.
 
 Another connection between numerical approximation and statistical inference arises in the 
framework of optimal recovery introduced by Micchelli and Rivlin \cite{micchelli1977survey,owhadi2019operator}
and its connection to Bayesian inference and GP regression through 
decision and game theory \cite{Wald:1945, VNeumann28}.
Optimal recovery was initially used for solving linear PDEs \cite{Harder:1972, Duchon:1977, Owhadi:2014}, but was extended to nonlinear PDEs in \cite{chen2021solving} and to general computational graph completion problems in \cite{owhadi2022computational} where the connection between 
optimal recovery and the GP perspective on solving PDEs is made explicit.
Finally, we mention the recent papers \cite{long2022autoip} and \cite{vadeboncoeur2023fully} 
where numerical errors are analyzed as  Bayesian posterior measures.
Further details about the connection between optimal recovery, decision theory, and GPs 
can be found in \Cref{appendix:optimal-recovery}.

\subsection{Notation and Preliminaries} \label{subsec:notation}
We use $|\cdot|$ to denote the finite-dimensional Euclidean norm.
Since $\mX$ is  Hilbertian,
all elements of the dual space $\mX^\ast$ may be 
Reisz-represented by elements of $\mX$ itself; if $\psi \in \mX^\ast$ then
we write $\psi^\ast \in \mX$ for its Reisz-representer.
Likewise, if $\theta \in \mX$ then we write $\theta^\ast$ for 
the dual element  it Reisz-represents. Throughout we will write $B_r(u) \subset \mX$ to 
denote the ball of radius $r \ge 0$ centered at $u$.

We give a brief summary of the notation from Gaussian measure theory 
needed for this paper; we follow \cite[Section 3]{hairer2009introduction} and the
reader seeking more details may consult \cite{bogachev1998gaussian}.
We say that a measure $\mu \in \PP(\mX)$
is a Gaussian measure (process) on $\mX$ if and only if for any $\psi \in \mX^\ast$, the pushforward
measure $ \mu \circ \psi^{-1} =: \psi_\sharp \mu  \in \PP(\R)$ is a Gaussian measure. 
Henceforth we write $\mu = N(m, \mK)$ to denote a Gaussian measure in $\PP(\mX)$
with mean $m \in \mX$ and covariance operator $\mK:\mX \to \mX $.
Whenever $m=0$ we say $\mu$ is a centered Gaussian measure. 
Note that $\mK$ is necessarily compact, indeed it is trace-class, and we may define the symmetric operator  $\mK^{\frac12}$ by spectral calculus;
operator $\mK^{-\frac12}$ can also be densely defined on $\mK^{1/2} \mX.$
Indeed, associated to a centered Gaussian measure $\mu = N(0, \mK)$,
we identify its Cameron-Martin space 
$\mH(\mu) := \mK^{1/2} \mX$ which is Hilbertian 
with corresponding inner product
\begin{equation*}\label{Cameron-Martin-Inner-Product}
  \langle u ,v \rangle_{\mH(\mu)} := \langle \mK^{-1/2} u, \mK^{-1/2} v \rangle_\mX, \qquad \forall u,v \in \mH(\mu); 
\end{equation*}
we write $\| \cdot \|_{\mH(\mu)}$ for the induced norm. 
The Cameron-Martin space is a Reproducing Kernel Hilbert Space (RKHS) if
pointwise evaluation is defined in $\mH(\mu);$ the kernel of the RKHS
is the covariance function associated with $\mu$; see \cite[Sec~2.3]{van2008reproducing}.
For any infinite-dimensional Gaussian measure, it is always true that $\mu\bigl(\mH(\mu)\bigr)=0;$ in contrast, by construction, $\mu(\mX)=1.$
Furthermore $\mH(\mu)$ is compactly embedded into $\mX.$

We will also review some preliminary definitions and results for
conditional measures
identified via a mapping as these ideas are central to our study. 
Our reference for this material is \cite[Sec.~10.4]{bogachev-measuretheory-II}.
Let $\mX, \mY$ be separable Hilbert spaces with $\mB(\mX), \mB(\mY)$ denoting their
respective Borel $\sigma$-algebras together with a measure $\nu \in \PP(\mX)$.
Consider a $(\mB(\mX), \mB(\mY))$--measurable
map  $T: \mX \to \mY$. We then have the following definition of a system of
conditional measures of $\nu$ generated by the mapping $T$:

\begin{definition}
  A function  $(A, y) \mapsto \nu^y(A)$ is a \emph{system of conditional measures} for $\nu$ with respect to the map $T$ if:

  \begin{enumerate}[label=\it (\alph*)]
  \item for every fixed $y \in \mY$ the function $\nu^y \in \PP(\mX)$;
  \item for every fixed $A \in \mB(\mX)$ the function $y \mapsto \nu^y(A)$ is
    measurable with respect to $\mB(\mY)$ and  $T_\sharp\nu$-integrable;
  \item for all $A \in \mB(\mX)$ and $E \in \mB(\mY)$ it holds that
    \begin{equation*}
      \nu(A \cap T^{-1}(E)) = \int_E \nu^y(A) T_\sharp \nu(\dd y).
    \end{equation*}
  \end{enumerate}
\end{definition}

We also use the alternative notation $\nu( \dd \xi \mid T(\xi) =y)$ 
to denote the system of conditional measures in the above definition;
this notation succintly captures  what is behind the definition.
The next result is a consequence of
\cite[Lem.~10.4.3 and Cor.~10.4.10]{bogachev-measuretheory-II}:

\begin{proposition}
\label{p:citethis}
  Consider the above setting and suppose $T: \mX \to \mY$ is $\nu$-measurable.
  Then it holds that:
  \begin{enumerate}[label=\it (\alph*)]
  \item there exists a system of conditional measures $\nu^y$ for $\nu$
    with respect to the map $T$;
  \item the conditional measures $\nu^y$ are essentially unique, i.e., there exists
    a set $Z \in \mB(\mY)$ so that $T_\sharp \nu(Z) = 0$ and the $\nu^y$ are unique for all
    $y \in \mY \setminus Z$ (i.e., essentially unique); 
    
  \item for $T_\sharp \nu$-a.e. $y$ the measures $\nu^y$ concentrate on $T^{-1}(y)$, i.e.,
    $\nu^y(\mX \setminus T^{-1}(y)) = 0$.
  \end{enumerate}
\end{proposition}

\begin{remark}
In most of this paper we consider $T=G$ where $G$ is defined in
(\ref{eq:G}b); thus $\mY$ is finite-dimensional. However we do make 
some theoretical observations and remarks about the more general 
setting, which includes infinite-dimensional $\mY.$
\end{remark}

\subsection{Outline} \label{subsec:overview}

In \Cref{sec:posterior-analysis} we analyze the posterior measure,
and limits as $\beta \to 0$. \Cref{sec:MAP-analysis} is devoted to
the modes, or MAP estimators, associated with the family of posterior measures,
and their $\beta \to 0$ limit. In \Cref{sec:sampling} we discuss algorithms
to sample the posterior measures, exploiting the special structure of the
observations and the decomposition of posterior measures. Finally, we give our conclusions in \Cref{sec:discussion}.
Proofs of various technical results are collected in the appendix.

\section{Analysis of Posterior and Conditional Measures}\label{sec:posterior-analysis}

In this section we study the posterior measures $\mu^\by_{\beta}$, 
and the conditionals $\mu^\by_{0}$. In  \Cref{subsec:posterior-convergence} 
we prove a form of convergence, suitably defined, of $\mu^\by_{\beta}$
to $\mu^\by_{0}$. \Cref{subsec:conditional-decomposition} studies decompositions of the
conditionals and posteriors respectively into the convolution of finite-dimensional non-Gaussians with an infinite-dimensional Gaussian part.

\subsection{Convergence of Posterior Measures to Conditionals}\label{subsec:posterior-convergence}
In this subsection we show that in the limit $\beta \to 0$ the posterior measures
$\mu^\by_\beta$ converge to the conditional measures $\mu^\by_0$
in an appropriate sense. 
We start by identifying conditions that ensure that
the family of posterior measures $\mu^\by_\beta$ are well-defined for
$\beta>0$. To this end consider the set-up of \Cref{ssec:BIP}. 
We formulate the BIP of determining $u|\by$, from
\eqref{eq:G}, under the following assumptions:

  \begin{assumption}\label{assumption:F-assumptions}
Assume that $u \sim \mu$, $\bzeta \sim \pi_\beta :=N(0,\beta^2 I)$ and $u, \bzeta$ are independent. Assume further that the map $F: \R^N \to \R^M$ is finite 
at some point $\bz' \in \R^N$ and that $F$ is locally Lipschitz, 
i.e., for every $r >0$ there exists $L(r) > 0$ such that
      \begin{equation*}
      \| F(\bz_1) - F(\bz_2) \|_2 \le L(r) \| \bz_1 - \bz_2 \|_2 
       \qquad \forall \bz_1, \bz_2 \in B_r(0).
    \end{equation*}
 \end{assumption}

Note that, since $F$ is finite at one point $\bz' \in \R^N$ then this assumption
implies that $F$ is locally bounded from above, i.e., for every $r >0$ there exists
      $M(r) > 0$ such that
      \begin{equation*}
      \| F(\bz) \|_2 \le M(r) \qquad   \forall \bz \in B_r(0).
    \end{equation*}
Recalling definition (\ref{eq:G}b), we have:
    
    \begin{lemma} \label{lem:citethis}
      Let \Cref{assumption:F-assumptions} hold and consider 
      the BIP for $u|\by$ defined via \eqref{eq:G}. 
      Then, for every $\beta>0$, the posterior distribution $\mu^\by_\beta$
      is given by \eqref{main-problem-RN}. Furthermore, the map $G$ defines a unique 
      (up to equivalence) system of conditional measures of $\mu$, denoted
      $\mu^\by_0:=\mu(\mathrm{d}u \mid G(u)=\by)$.
    \end{lemma}
    
    \begin{proof}
     \cite[Thm.~10]{dashti2017bayesian} establishes the result for $\beta>0.$
     The result for $\beta=0$ follows from \Cref{p:citethis}, using
     the fact that, under the stated assumptions on $F$, $G$ is continuous and hence $\mu-$measurable as a map from $\mX$ into $\R^M.$
     \end{proof}

We now consider the limit of the measures $\mu^\by_\beta$ as $\beta \to 0.$ It is
convenient to express the result in terms of the joint measure  $\PP(\dd u,\dd \by).$
This joint measure may be factored as $\PP(\dd u|\by)\PP(\dd \by)$ or as $\PP(\dd \by|u)\PP(\dd u).$ The latter necessarily involves a Dirac mass when $\beta=0$ and is not convenient to work with; we hence use the former factorization.

\begin{theorem}\label{prop:convergence-of-conditionals-and-pushforwards}
  Let \Cref{assumption:F-assumptions} hold. Then the measures $\mu^\by_\beta( \dd u) G_\sharp \mu \ast \pi_\beta (\dd \by)$
  converge weakly to $\mu^\by_0(\dd u) G_\sharp \mu(\dd \by)$
  as $\beta \to 0$. That is,
  $\forall f \in C_b(\mX \times \mY)$
  \begin{equation*}
    \lim_{\beta \to 0} \int_{\mY}  \int_\mX  f(u, \by)   \mu^\by_\beta(\dd u)
    G_\sharp \mu \ast \pi_\beta( \dd \by)
    = \int_{\mY} \int_{\mX} f(u, \by)   
\mu^\by_0(\dd u) G_\sharp \mu (\dd \by).
  \end{equation*}
\end{theorem}

\begin{proof}
It will be helpful to extend $\pi_\beta$ to a measure on $\mX \times \R^M$
by defining $\pi_\beta':=\delta_0 \times N(0,\beta^2 I).$ With this
notation we note that
$$\mu^\by_\beta(\dd u)
    G_\sharp \mu \ast \pi_\beta( \dd \by)=( Id \times G)_\sharp \mu \ast\pi_\beta'(\dd u, \dd \by),$$
    and that
$$\mu^\by_0(\dd u) G_\sharp \mu (\dd \by)=(Id \times G)_\sharp \mu(\dd u, \dd \by).$$
The desired result thus reduces to proving that,
  $\forall f \in C_b(\mX \times \mY),$
  \begin{equation*}
    \lim_{\beta \to 0} \int_{\mY}  \int_\mX  f(u, \by)  
    ( Id \times G)_\sharp \mu \ast\pi_\beta'(\dd u, \dd \by)
    = \int_{\mY} \int_{\mX} f(u, \by)   
(Id \times G)_\sharp \mu(\dd u, \dd \by).
  \end{equation*}
  Noting that $\pi_\beta'$ converges weakly to a Dirac at the origin in
  $\mX \times \R^M$ as $\beta \to 0$ gives the desired result.
\end{proof}

\begin{remark}
    We note that the above result can be interpreted as an ``almost'' weak convergence result 
    for the posterior measures $\mu^\by_\beta$. More precisely, take $f(u, \by) = g(u)h(\by)$
    where $g \in C_b(\mX)$ and $h \in C_b(\mX)$ is a continuous approximation to 
    $\frac{1}{\mu(B_\epsilon(\by'))} \mbf{1}_{B_\epsilon(\by')}$ for some fixed $\by' \in \R^M$
    and $\epsilon >0$. Then 
    \Cref{prop:convergence-of-conditionals-and-pushforwards} tells us that 
    the expectation of $g$ with respect to $\mu^{\by'}_\beta$ converges to the 
    conditional expectation with respect to $\mu^{\by'}_0$ so long as we 
    average $\by$  in a ball with arbitrarily small but positive  radius $\epsilon$ around $\by'$.
\end{remark}

\subsection{Finite-Dimensional Representation of Conditional and Posterior Measures}\label{subsec:conditional-decomposition}

The finite-dimensional representation of the conditionals is
analogous to the family of  representer theorems for kernel methods \cite[Sec.~4.2]{smola1998learning}, generalized to the probabilistic
setting, and is stated as 
\Cref{prop:conditional-measure-decomposition} below. 
To understand this proposition, we first recall a classic lemma 
pertaining to conditioning Gaussian measures on direct sums of Hilbert spaces,
and a corollary thereof.

Let 
$\mX = \mX_1 \oplus \mX_2$ where $\mX_1, \mX_2$ are separable Hilbert spaces
and let $\mu$ be a Gaussian measure on $\mX$  and
 $\Pi_i: \mX \to \mX_i$ denote the natural projection  onto $\mX_i$.  
Then by \cite[Lem.~4.3]{stuart2010inverse} and \cite{owhadi2015conditioning} (see also \cite[Chap.~17.8]{owhadi2019operator})
we have that the conditional measure of $\mu$ with respect  
to the maps $\Pi_i$ is also Gaussian and can be characterized explicitly: 


\begin{lemma}\label{prop:conditional-Gaussian-direct-sum}
  Let $\mu = N(m, \mK) \in \PP(\mX)$ where $\mX = \mX_1 \oplus \mX_2$ as above.
  Write $m=(m_1, m_2)$ for the mean and let $\mK$ be the positive definite covariance operator
  and define $\mK_{ij} = \Pi_i \mK \Pi^\ast_j$. Write $\mu^{x_1}$ for the system of conditional
  measures of $\mu$ with respect to $\Pi_1$. 
  Then for $(\Pi_1)_\sharp \mu$-a.e. $x_1 \in \mX_1$
  it holds that  $\mu^{x_1} = \delta_{x_1} \otimes N(m^{x_1}, \mK_{2|1})$
where $N(m^{x_1}, \mK_{2|1})$ is a Gaussian measure on $\mX_2$ with
mean   $m^{x_1} = m_2 + \mK_{21} \mK_{11}^{-1}( x_1 - m_1)$ and covariance operator $\mK_{2|1} = \mK_{22} - \mK_{21} \mK_{11}^{-1} \mK_{12}.$
\end{lemma}

The following corollary may be deduced by applying 
Lemma~\ref{prop:conditional-Gaussian-direct-sum} to the measure
  $\mu \otimes \bphi_\sharp \mu$  on the product space $\mX \times \R^N$,
using the fact that $\bphi_\sharp \mu = N( 0, \Theta)$ and that the tensor 
product of two Gaussian measures is also Gaussian:

\begin{corollary}\label{lem:linear-Gaussian-conditional}
Suppose $\mu = N(0, \mK)$ with $\mK$ a trace-class covariance operator on $\mX$.  Consider the map $\bphi = (\phi_1, \dots, \phi_N) \in (\mX^\ast)^{N}$ and
  define the vector $\btheta=\{\theta_i\}_{i=1}^N \in \mX^N$ and the symmetric
  matrix $\Theta=\{\Theta_{ij}\}_{i,j=1}^N \in \R^{N \times N}$ with entries
  \begin{equation}\label{def:theta-vector-and-matrix}
  \theta_i := \mK \phi_i^\ast, \qquad \text{and} \qquad 
  \Theta_{ij} := \phi_i ( \mK \phi_j^\ast).
\end{equation}
Consider the system of conditional measures $\mu^\bz \equiv \mu\bigl(\dd u \mid\bphi(u)=\bz\bigr)$. If $\Theta$ is invertible then 
$ \mu^\bz=N(u^\bz, \mK^\bphi)$ where  
\begin{subequations}\label{def:u-dagger-and-K-phi}
\begin{align}
  u^\bz &=\btheta^T \Theta^{-1} \bz:= \sum_{i,j=1}^N (\Theta^{-1})_{ij} \theta_i z_j, \quad\\
 \mK^\bphi &= \mK-\btheta^T \Theta^{-1}\btheta^\ast:= \mK -  \sum_{i,j=1}^N (\Theta^{-1})_{ij} \theta_i \theta_j^\ast.
  \end{align}
\end{subequations}
\end{corollary}

\begin{remark}\label{rem:Gaussian-random-mean-shift}
  We often consider the vector of functions $\bchi := \Theta^{-1} \btheta \in \mX^N$; the entries $\vphi_i$ of $\bchi$ are referred to as the {\it Gamblets} in the parlance of \cite{owhadi2019operator}.
We can then write $u^\bz = \bchi^T \bz$ and
  refer to $\bchi^T : \R^N \to \mX$ as the Gamblet reconstruction map.
  In the following it is useful to define $\mu^\bphi:=N(0,\mK^\bphi)$
  and $\eta = N( 0, \Theta)$, noting that the latter is the distribution
  of $\phi_\sharp \mu.$
  Now notice that the measure $\mu$ can be reconstructed as the convolution
  $\mu=\mu^\bphi \ast \bchi^T_\sharp \eta.$ Crucial to this fact is that the $u^\bz$
  depends on $\bz$ whereas $\mK^\bphi$ does not, it only depends on the linear
  map $\bphi$ and  not the vector $\bz$, and that $\bz \sim \eta$ under $\mu.$
\end{remark}
Building on this remark we have
the following useful factorization of the
conditional $\mu^\by_0$ which is one of our main theoretical contributions.

 \begin{theorem}\label{prop:conditional-measure-decomposition}
   Suppose \Cref{assumption:F-assumptions} holds and that  \Cref{lem:linear-Gaussian-conditional} is  satisfied. Then $\mu = \mu^\bphi \ast \bchi^T_\sharp \eta$ and
   $\mu^\by_0 = \mu^\bphi \ast \bchi^T_\sharp \eta^\by_0$, where
   $\eta^\by_0:= \eta( \dd \bz \mid F(\bz) = \by)$ is the system of  conditionals of
   $\eta = N( 0, \Theta)$ with respect to the map $F$.
 \end{theorem}

 \begin{proof} Let $u \sim \mu.$
 Conditional on $\bphi(u)=\bz$ the distribution of $u$ is $N(u^\bz,\mK^\bphi),$
 by \Cref{lem:linear-Gaussian-conditional}. In the absence of observations,
 $\bz \sim \eta$ and then $u^\bz \sim \bchi^T_\sharp \eta.$ When conditioned on $F(\bz)=\by$, however, we obtain
 $\bz \sim \eta^\by_0$ and $u^\bz \sim \bchi^T_\sharp \eta^\by_0.$ Because $\mu^\bphi$ is independent of $\bz$ the two results follow
 by the properties of convolutions of measures.
 \end{proof}

We may now generalize
\Cref{prop:conditional-measure-decomposition}
to the setting $\beta  > 0$; we show that the posterior measures in \eqref{main-problem-RN} can
be decomposed as the convolution of a finite-dimensional (in general) non-Gaussian measure with an independent centered Gaussian measure. 
The result may also be viewed as a generalization of  
\Cref{lem:linear-Gaussian-conditional}
to nonlinear measurements. 
This theorem is the second major theoretical contribution of our work.

\begin{theorem}\label{prop:posterior-decomposition}
Suppose \Cref{assumption:F-assumptions} holds and that \Cref{lem:linear-Gaussian-conditional}
is satisfied.
  Let $\mu^\by_\beta$ be as in \eqref{main-problem-RN} and let $\Lambda$ denote the
Lebesgue measure. 
  Then $\mu^\by_\beta = \mu^\bphi \ast \bchi^T_\sharp \eta^\by_\beta $
   where 
  $\eta^\by_\beta \in \PP(\R^N)$ has Lebesgue density 
  \begin{equation*}
    \begin{aligned}
    \frac{\dd \eta^\by_\beta}{\dd \Lambda}(\bz) & =
    \frac{1}{\varpi_\beta(\by)} \exp \left( - \frac{1}{2\beta^2} | F(\bz) - \by |^2  - \frac{1}{2} \bz^T \Theta^{-1} \bz \right), \\
    \varpi_\beta(\by) & := \int_{\R^N} \exp \left( - \frac{1}{2\beta^2} | F(\bz) - \by |^2  - \frac{1}{2} \bz^T \Theta^{-1} \bz \right)  \Lambda( \dd \bz).
  \end{aligned}
  \end{equation*}
\end{theorem}

  \begin{proof}
    Recall $\pi_\beta$ from \Cref{assumption:F-assumptions} and,
  for any measure $\pi$ on a vector space,
  let $ \pi( \dd \by + \bmm )$  denote the shift of $\pi$ by a vector $\bmm$.
Consider the measure
  $\nu(\dd u, \dd \by)  := \mu( \dd u) \pi_\beta(\dd \by + F(\bphi(u)))$. By Bayes' rule
  the posterior measures $\mu^\by_\beta(\dd u) \otimes \delta_\by (\dd y)$ are precisely the conditionals of $\nu$
  with respect to the projection $\Pi: \mX \times \R^M \to \R^M$, i.e.,
  \begin{equation*}
    \nu( A \cap \Pi^{-1}(E)) = \int_E \big(\mu^\by_\beta \otimes \delta_\by\big)(A)
    \big(F_\sharp(\bphi_\sharp \mu) \ast \pi_\beta \big) (\dd \by), \qquad A \in \mB(\mX \times \R^M), E \in \mB(\R^M).
  \end{equation*}
  Further consider the measure
  $\tilde{\eta}( \dd \bz, \dd \by) := \bphi_\sharp \mu(\dd \bz) \pi_\beta( \dd \by + F(\bz))$.
  Applying Bayes' rule once again we identify $\eta^\by_\beta$ as the conditionals of
  $\tilde{\eta}$  with respect to the projection $\tilde{\Pi}: \R^N \times \R^M \to \R^M$,
  \begin{equation*}
  \begin{aligned}      
    \tilde{\eta} &( B \cap \tilde{\Pi}^{-1}(E)) \\
    & = \int_E \big(\eta^\by_\beta \otimes \delta_\by\big)(B)
    \big(F_\sharp(\bphi_\sharp \mu) \ast \pi_\beta \big) (\dd \by), \qquad B \in \mB(\R^N \times \R^M), E \in \mB(\R^M).
      \end{aligned}
  \end{equation*}
By \Cref{lem:linear-Gaussian-conditional} we have  that
  $(I \times \bphi)_\sharp \mu (\dd u, \dd \bz)  = \mu^\bphi(\dd u +   \bchi^T \bz ) \bphi_\sharp \mu(\dd \bz)$.
  Now define the measure
  $\tilde \nu := ( I \times \bphi)_\sharp \mu (\dd u, \dd \bz) \pi_\beta( \dd \by + F(\bz))
  \in \PP( \mX \times \R^N \times \R^M)$.  We then have, by the above arguments and
  \Cref{rem:Gaussian-random-mean-shift},
  \begin{equation*}
    \begin{aligned}
      \tilde \nu (\dd u, \dd \bz, \dd \by)
      & = \mu^\bphi (\dd u +  \bchi^T \bz ) \bphi_\sharp \mu(\dd \bz) \pi_\beta( \dd \by + F(\bz)) \\
      & = \mu^\bphi(\dd u +  \bchi^T \bz) \tilde{\eta}( \dd \bz, \dd \by) \\
      & = \mu^\bphi(\dd u + \bchi^T \bz)  \eta_\beta^\by (\dd \bz) 
      \big( F_\sharp(\bphi_\sharp \mu) \ast \pi_\beta \big) (\dd \by). 
\end{aligned}
\end{equation*}
Now observe that 
 $\nu = T_\sharp \tilde \nu$ 
 where $T: (u, \bz, \by) \mapsto (u, \by)$ so that we have the desired identity
 \begin{equation*}
   \nu( \dd u, \dd y) = \big( \mu^\bphi \ast \bchi^T_\sharp \eta^\by_\beta ) (\dd u)
   \big( F_\sharp(\bphi_\sharp \mu) \ast \pi \big) (\dd \by).
 \end{equation*}
  {}
\end{proof}

\section{Modes of  Posterior and Conditional Measures}\label{sec:MAP-analysis}

In this section we analyze the modes of the posteriors $\mu^\by_\beta$ (i.e., the MAP estimators)
and the conditionals $\mu^\by_0$.
\Cref{subsec:MAP-estimators} defines the mode of the posterior; subsection
\Cref{subsec:conditional-mode} defines the mode of the conditional; and
\Cref{subsec:MAP-convergence} considers the $\beta \to 0$ limit of the posterior modes.

\subsection{Modes of Measures}\label{subsec:MAP-estimators}
We recall the notion of the mode of a measure employed  
in \cite{dashti2013map}:

\begin{definition}\label{def:MAP}
  Consider a measure $\nu \in \PP(\mX)$.
  Any point $u^\dagger \in \mX$ is a \emph{mode} of $\nu$ if it satisfies
  \begin{equation*}
    \lim_{r \to 0} \frac{\nu(B_r(u^\dagger))}{ \sup_{u \in \mX} \nu(B_r(u))} = 1.
  \end{equation*}
\end{definition}

This formalizes the idea of defining the mode as the centre of
a small ball of maximal probability, in the limit of vanishing radius.
The modes of the posterior measures $\mu^\by_\beta \in \PP(\mX)$ defined in \eqref{main-problem-RN} are referred to as MAP estimators.
The next proposition follows directly from \cite[Cor.~3.10]{dashti2013map}
which allows us to characterize the MAP estimators of $\mu^\by_\beta$
via the optimization problem \eqref{MAP-optimization}.
We emphasize that local minimizers of \eqref{MAP-optimization}
may not be unique, but that a global minimizer exists provided 
that $F \circ \bphi$ is continuous
on $\mX$ \cite{dashti2013map}.

\begin{theorem}\label{General-MAP-characterization}
  Suppose $\mu = N(0, \mK)$, $\mu^\by_\beta$ is defined as in \eqref{main-problem-RN} with $\beta >0$ and
  $\by \in \R^M$, and
  the map $F: \R^N \to \R^M$ satisfies 
  \Cref{assumption:F-assumptions}. Define the
{\it Onsager-Machlup (OM)}
    functional $J^\by_\beta: \mX \to [0,\infty]$ by
     \begin{equation*}
      J^\by_\beta(u) := \left\{ 
        \begin{aligned}
          & \frac{1}{2\beta^2} | F(\bphi(u)) - \by |^2 + \frac{1}{2} \| u \|_{\mH(\mu)}^2,
          && \text{if } u \in \mH(\mu), \\ 
          & + \infty, && \text{if } u \in \mX \setminus \mH(\mu). 
      \end{aligned}
\right.
    \end{equation*}
Then a point $u^\by_\beta \in \mX$ is a MAP estimator  for $\mu^\by_\beta$, according to \Cref{def:MAP},
if and only if it is a  minimizer of $J^\by_\beta$ over $\mX.$
  \end{theorem}

\begin{proof}
To apply the stated corollary define $\Phi(u):=\frac{1}{2\beta^2} | F(\bphi(u)) - \by |^2.$
Notice that $\Phi$ is bounded below uniformly on $\mX$, is bounded
above on bounded sets in $\mX$ and is Lipschitz on bounded sets in $\mX.$
Then the result follows by a direct application of \cite[Cor.~3.10]{dashti2013map}.
\end{proof}

We now further characterize
 MAP estimators of $\mu^\by_\beta$  via a
 representer theorem for the minimizers of OM 
 functionals. This theorem constitutes our main result towards the 
 finite-dimensional characterization of MAP estimators.
 
\begin{theorem}\label{lem:MAP-decomposition}
  Suppose that the conditions of \Cref{General-MAP-characterization} are satisfied. 
  Then $u^\by_\beta$ is a MAP estimator for $\mu^\by_\beta$
  if $u^\by_\beta = \bchi^T \bz^\by_\beta$ and  $\bz^\by_\beta \in \R^N$  solves
  \begin{equation}
  \label{eqn: posterior modes def}
    \minimize_{\bz \in \R^N} \frac{1}{2\beta^2} | F(\bz) - \by |^2 + \frac{1}{2} \bz^T \Theta^{-1} \bz.
  \end{equation}
\end{theorem}
\begin{proof}
  By \Cref{General-MAP-characterization} $u^\by_\beta$ is a MAP estimator for $\mu^\by_\beta$
  if it is a minimizer of the OM functional.
  Applying  \cite[Prop.~2.3]{chen2021solving} to characterize the minimizers of the OM functional yields the desired result.
  {}
\end{proof}

\begin{remark}
Let $\beta>0.$ Note that solutions of the optimization problem \eqref{MAP-optimization} (i.e., 
minimizers of the OM functional) are necessarily in $\mH(\mu)$; samples 
from the posterior $\mu^\by_\beta$ given by \eqref{main-problem}, 
however, are almost surely 
not in $\mH(\mu)$ because the posterior is absolutely
continuous with respect to the prior $\mu$ and $\mu\bigl(\mH(\mu)\bigr)=0.$
Simply put, we need 
$\mX$ to be sufficiently regular so that $\phi_i \in \mX^\ast$ for the 
probabilistic formulation to make sense, however, the optimization problems \eqref{MAP-optimization} 
and \eqref{MAP-optimization2}
require the $\phi_i$ to be bounded 
and linear functionals on both $\mH(\mu)$
and $\mX$.

This observation has important implications in the context of the 
 GP-PDE solver of \Cref{subsec:motivation}.
In order to apply the optimization approaches 
\eqref{MAP-optimization} or \eqref{MAP-optimization2} to solving 
PDEs as in \cite{chen2021solving}, it is necessary that pointwise
evaluation of all derivatives appearing in the PDE is possible in
$\mH(\mu)$. To apply the probabilistic (Bayesian) approach
\eqref{main-problem} or \eqref{main-problem2} 
to the same problem, pointwise evaluation of all derivatives 
appearing in the PDE is needed over the support of $\mu$, i.e., the space $\mX$. Thus
the probabilistic approach places a more stringent requirement on the Gaussian prior
measure $\mu$ than does the optimization approach. 
\end{remark}

\subsection{Modes of Conditional Measures}\label{subsec:conditional-mode}
Here we define a novel notion of a mode for a conditional measure.
We develop a theorem applicable for general maps $T$ with respect to which 
conditional measures are defined and specified to the 
case $T=G$, with $G$ given by (\ref{eq:G}b), in a corollary.

\begin{definition}\label{def:conditional-mode}
 Consider  separable
 Hilbert spaces $\mX, \mY$,
 a measure $\nu \in \PP(\mX)$, and a map $T: \mX \to \mY$.
 Fix a point $y \in \supp T_\sharp \nu$.
  Then any point $u^\dagger \in T^{-1}(y)$ that satisfies
  \begin{equation*}
    \lim_{r \to 0 }  \frac{\nu( B_r(u^\dagger))}{ \sup_{u \in T^{-1}(y)} \nu(B_r(u))} = 1,
  \end{equation*}
  is  a \emph{conditional mode} of $\nu( \dd u  \mid T(u) = y)$.
\end{definition}

  The above definition of the conditional mode is a natural extension of \Cref{def:MAP} and modifies 
  that definition by restricting the feasible set of $u^\dagger$ to the subset $T^{-1}(y) \subseteq \mX$. Below we
  show that this definition leads to a natural characterization of conditional modes of Gaussian measures
  via constrained optimization problems, this is the conditional analog of \Cref{General-MAP-characterization} and constitutes one of our main theoretical contributions in the paper.

\begin{theorem}\label{prop:conditional-mode-optimization}
Let $\mX, \mY$ be separable Hilbert spaces and suppose 
$T: \mX \to \mY$ is continuous. Consider $\mu = N(0, \mK) \in \PP(\mX)$ with Cameron-Martin space $\mH(\mu).$
  Fix a
  point $y \in  T(\mH( \mu)) \cap \supp T_\sharp \mu$, assuming the 
  intersection is non-empty. Then $u^y$ is a conditional mode of $\mu( \dd u \mid T(u) = y)$
  if and only if it solves the optimization problem
  \begin{equation}\label{conditional-map-optimization-problem}
    \begin{aligned}
      \minimize_{u \in \mX} \quad \| u\|_{\mH(\mu)} \quad \st  \quad  T(u) = y.
    \end{aligned}
    \end{equation} 
    \end{theorem}
The proof follows by adapting the proof techniques 
of \cite[Cor.~3.10]{dashti2013map}
to our definition of a conditional mode. 
The details are summarized  in 
Appendix~\ref{appendix:Proof-of-conditional-mode-optimization} for brevity.

\begin{remark}
The preceding theorem requires both that
$y \in \supp T_\sharp \mu$ and that $y \in  T(\mH( \mu)).$ 
The first condition is natural: we want the data to have
arisen, in principle, from a map $T$ applied to the
realization of the measure $\mu.$
The second condition, however, says that it must also be
realized as an application of the map $T$  to a point in the Cameron-Martin
space $\mH( \mu).$ Recall that $\mu(\mH( \mu))=0.$
Requiring both of these conditions to hold leads to restrictions on the map $T$.

Consider the following example of a Gaussian measure
from \cite{dashti2017bayesian}.
Assume a centered Gaussian measure $\mu$ with a covariance operator which is the inverse of $-\frac{d^2}{dx^2}$ on $I:=(0,1)$, with
homogeneous Dirichlet boundary conditions; this is a compact operator from $L^2(I)$ into
itself. Thus $\mu$ is the 
Brownian bridge and we may take $\mX=H^s(I)$, for any $s<\frac12$ since
all such  Sobolev spaces are in the support of $\mu$. Furthermore
any draw from $\mu$ is almost surely \emph{not} an element of $H^s(I)$ for
any $s \ge \frac12.$ In particular the Cameron-Martin space is
$H^1_0(I)$ and $\mu(H^1_0(I))=0.$ 
Now define $t:\R \to \R$ by $t(u)={\rm min}(1,u)$
and  $T:\mX \to \mX$ by $T(u)(x):=t\bigl(u(x)\bigr).$
Applying such a function $t(\cdot)$ pointwise to any draw from $\mu$ results, 
almost surely, in a function with no more than $s<\frac12$ weak derivatives
in $L^2(I)$. Such a function cannot simultaneously
be the image under a globally Lipschitz $T(\cdot)$ 
of an element of $H^1_0(I).$ Thus the preceding theorem
cannot be applied.

On the other hand, working with the same measure $\mu$,
taking $\mY=\R$ and $T(u)=u(\frac12)$ it follows from
the previous regularity discussions, and the properties
of Brownian bridge at any point in the open interval $I$,
that any $y \in \R$ is also in
$T(\mH( \mu)) \cap \supp T_\sharp \mu.$
Thus the theorem can be applied.
\end{remark}

Noting the ideas underlying the preceding remark,
the following corollary of Theorem \ref{prop:conditional-mode-optimization}
is immediate, noting the finite-dimensionality of the image of 
$T:=F \circ \bphi.$

\begin{corollary}\label{General-MAP-characterization2}
Consider $\mu = N(0, \mK) \in \PP(\mX)$ with Cameron-Martin space $\mH(\mu),$
and map $F: \R^N \to \R^M$ satisfying \Cref{assumption:F-assumptions}.
  Suppose  $\mu^\by_0$ is defined as in \Cref{lem:citethis} for some $\by \in F\Bigl(\bphi\bigl(\mH(\mu)\bigr)\Bigr) \subseteq \R^M$ and with $\phi_i \in \mX^\star$. Then a point $u^\by_0 \in \mX$ is a conditional mode  for $\mu^\by_0$, according to \Cref{def:conditional-mode},
    if and only if it is a  minimizer of the constrained optimization problem
    \begin{equation}\label{conditional-mode-problem}
    \begin{aligned}
      \minimize_{u \in \mX} \quad \| u\|_{\mH(\mu)} \quad \st  \quad  F\bigl(\bphi(u)\bigr) = \by.
    \end{aligned}
  \end{equation}
\end{corollary}

Using the representer theorem \cite[Prop.~2.3]{chen2021solving},
we can further characterize the
conditional modes $u^{\by}_0$ via a finite-dimensional optimization problem.
We recall this result for convenience.

\begin{proposition}\label{prop:representer-theorem-conditional-mode}
Suppose \Cref{General-MAP-characterization2}
is satisfied.  Then every conditional mode $u^\by_0$  of $\mu^\by_0$ 
can be written as  $u^\by_0 = \bchi^T \bz^\by_0$ where $\bz^\by_0$ is a solution of
  \begin{equation*}
      \minimize_{\bz \in \R^N} \quad \bz^T \Theta^{-1} \bz \quad \st  \quad  F(\bz) = \by.
  \end{equation*}
\end{proposition}

\subsection{Convergence of MAP Estimators to Conditional Modes}\label{subsec:MAP-convergence}

Finally, we establish the convergence of the MAP estimators 
$u^\by_\beta$ to the conditional modes $u^\by_0$ 
 in the setting where
$T=G$, with $G$ given by (\ref{eq:G}b). 

\begin{theorem}\label{prop:map-convergence-to-conditional-mode}
Consider $\mu = N(0, \mK) \in \PP(\mX)$ with Cameron-Martin space $\mH(\mu),$
and  a map $F: \R^N \to \R^M$ satisfying \Cref{assumption:F-assumptions}.
  Fix a point $\by \in G\bigl(\mH(\mu)\bigr)$ and consider the posteriors $\mu^\by_\beta$ and
  their MAP estimators $u^\by_\beta$, along with 
  the conditional measures $\mu^\by_0$ and their conditional modes $u^\by_0$.
  Then for any sequence of $\beta \to 0$ there exists a subsequence $\beta_n \to 0$
  so that $u^\by_{\beta_n}$ converges to a conditional mode $u^\by_0$.
\end{theorem}

\begin{proof}
First define
$J_\beta(u):=\| u \|_{\mH(\mu)}^2
  + \frac{1}{\beta^2} | G(u) - \by |^2,$
recalling that $\mH(\mu)=\mK^{\frac12}\mX$ is compactly embedded into $\mX.$
Note that $u^\by_\beta$ is a minimizer of $J_\beta$ over $\mX$ and
 that $u^\by_0$ minimizes $\| u \|_{\mH(\mu)}$ in $G^{-1}(\by) \subseteq \mX$. Hence, it holds that
 \begin{equation}
\label{eq:willcite2}
\|  u^\by_\beta\|_{\mH(\mu)}^2 \le J_\beta(u^\by_\beta) \le J_\beta(u^\by_0)=\|  u^\by_0\|_{\mH(\mu)}^2. 
\end{equation}
Thus we have that 
  $ \| u^\by_\beta \|_{\mH(\mu)} \le \| u^\by_0 \|_{\mH(\mu)}$ for all $\beta>0.$ Since $\mH(\mu)$ is a compact subset of $\mX$ we have convergence of
$u^\by_\beta$ in $\mX$ to a limit $u_\ast \in \mH(\mu),$
  as well as weak convergence in $\mH(\mu)$, along a subsequence $\beta_n.$ 
It is immediate that $u_\ast \in G^{-1}(\by)$ as otherwise
  (along a further relabelled subsequence) there is $\epsilon>0$ and $N \in \mathbb{N}$ such that $J_{\beta_n}(u^\by_{\beta_n}) \ge \epsilon/\beta_n^2$
  for all $n \ge N$, which contradicts \eqref{eq:willcite2} for all $n$ such that $\beta_n$ is sufficiently small. To show that $u_\ast$
  is equal to a minimizer of $\| u \|_{\mH(\mu)}$ in $G^{-1}(\by)$ we assume
  for contradiction that $\| u_\ast \|_{\mH(\mu)}>\| u_0^\by \|_{\mH(\mu)}.$
  By \eqref{eq:willcite2} we have
  \begin{equation*}
\|  u^\by_{\beta_n}\|_{\mH(\mu)}^2 \le J_{\beta_n}(u^\by_{\beta_n}) \le J_{\beta_n}(u^\by_0)=\|  u^\by_0\|_{\mH(\mu)}^2<\|  u_\ast\|_{\mH(\mu)}^2. 
\end{equation*}
However, by lower semi-continuity of Hilbert space norms, we also have
  $${\rm liminf}_{\beta_n \to 0}\, \|  u^\by_{\beta_n}\|_{\mH(\mu)}^2 \ge 
  \|u_{\ast}\|_{\mH(\mu)}^2,$$
  giving the desired contradiction.
  {}
\end{proof}

\section{Algorithms} \label{sec:sampling}
In this section, we discuss algorithms to sample the posterior and conditional measures of Gaussian priors. According to \Cref{prop:conditional-measure-decomposition,prop:posterior-decomposition}, both measures can be represented by a convolution of a finite-dimensional measure that is possibly non-Gaussian, and an infinite-dimensional Gaussian measure that can be identified analytically. Our goal here is 
to exploit this structure to design efficient algorithms 
for simulation of the aforementioned
posterior and conditional measures as 
summarized in \Cref{ssec:1,ssec:2}.
In \Cref{sec:numerics} we present more concrete examples 
where posterior measures arising within the
GP-PDE methodology are simulated.

\subsection{Sampling Strategies for Posterior Measures ($\beta^2  >0$)}
\label{ssec:1}
The key idea behind our proposed numerical algorithms is the observation that
\Cref{prop:posterior-decomposition} enables the decomposition
$\mu^\by_\beta = \mu^\bphi \ast \bchi^T_\sharp \eta^\by_\beta$ where 
$\mu^\bphi = N(0, \mK^\bphi)$ is a Gaussian whose covariance 
operator is given by \eqref{def:u-dagger-and-K-phi}, in terms of the 
measurement operator $\bphi$ and the prior covariance matrix $\mK$. Thus, the 
measure $\mu^\bphi$ can be simulated via standard techniques for discretization and 
sampling of Gaussian processes and measures 
\cite{williams2007gaussian, betz2014numerical, snelson2007local}. Furthermore, the map $\bvphi$ 
(recall \Cref{rem:Gaussian-random-mean-shift}) is also defined using 
$\bphi$ and $\mK$ and so can be approximated via appropriate discretization.
It remains to simulate $\eta^\by_\beta$ which is, in general, non-Gaussian. We recall that 
\Cref{prop:posterior-decomposition} identifies $\eta^\by_\beta \in \PP(\R^N)$ via its Lebesgue density
  \begin{equation*}
    \frac{\dd \eta^\by_\beta}{\dd \Lambda}(\bz)  \propto
    \exp \left( - \frac{1}{2\beta^2} | F(\bz) - \by |^2  - 
    \frac{1}{2} \bz^T \Theta^{-1} \bz \right)\, .
  \end{equation*}
At this level any sampling algorithm of choice such as MCMC \cite{robert1999monte}, sequential Monte Carlo \cite{doucet2001introduction},
or variational inference \cite{blei2017variational} can be used to simulate samples from $\eta^\by_\beta$, leading 
to an algorithm for simulating posterior samples as summarized in \Cref{alg:MCMC}.
While this approach is accurate up to the discretization errors of $\mK^\bphi$ and $\bvphi$ 
and the convergence of the utilized sampling algorithms for $\eta^\by_\beta$, it has limited utility 
in the limit $\beta \to 0$ which is particularly important in the context of 
the GP-PDE solver of \Cref{subsec:motivation}. This is due to the well-understood 
phenomenon that as $\beta \to 0$ the measure $\eta^\by_\beta$ concentrates on 
the set $F^{-1}(\by)$ which may have very small prior measure, leading to poor 
convergence rates for sampling algorithms such as MCMC. 

\begin{algorithm}[htp]
\footnotesize
\caption{Recipe for generating  samples from $\mu^\by_\beta$ using MCMC on $\eta^\by_\beta$}\label{alg:MCMC}
\begin{algorithmic}[1]
\State {\bf Input:} Prior covariance $\mK$, maps $F, \bphi$, and $\beta > 0$
\State {\bf Output:} Samples $u_j \sim \mu^\by_\beta$
\State Discretize the operators $\mK^\bphi$ and $\bvphi$ as $\widehat{\mK}^\bphi$ and 
$\widehat{\bvphi}$ 

\For{$j=1, \dots,$ Number of samples }
\State Simulate $w_j \sim \mu^\bphi$ 
by setting $w_j = (\widehat{\mK}^\bphi)^{1/2} \xi_j$ 
where $\xi_j \sim N(0, I)$
\State Simulate $v_j \sim \eta^\by_\beta$ using MCMC (or similar algorithm)
\State Set $u_j = w_j + \widehat{\bvphi}^T v_j$
\EndFor
\end{algorithmic}
\end{algorithm}

Under the  conjecture that $\eta^\by_\beta$ approaches a Gaussian measure in 
the limit of large data and small noise, we propose to replace Step 6 of \Cref{alg:MCMC}
with a Gaussian approximation step at the mode; this is sometimes referred to  as the Laplace approximation 
to $\eta^\by_\beta$ \cite{kass1991laplace}. More precisely, letting $\bz^\by_\beta$ be a mode of $\eta^\by_\beta$
obtained by solving \eqref{eqn: posterior modes def}, we define the Gaussian measure 
\begin{equation}
  \label{eqn: sampling posteriors density Laplace approximation}
    \begin{aligned}
    &\frac{\dd \overline{\eta}^\by_\beta}{\dd \Lambda}(\bz) \propto \\
    &  
    \exp \left(  - \frac{1}{2 \beta^2}  (\bz - \bz^\by_\beta)^T 
    \left( 
    \nabla F(\bz^\by_\beta)^T \nabla F(\bz^\by_\beta) + 
    D^2 F(\bz^\by_\beta) (F(\bz^\by_\beta) - \by)  
    \right)
    (\bz - \bz^\by_\beta)
    \right)\, .
  \end{aligned}
  \end{equation}
  The above Laplace approximation leads to an efficient sampling algorithm for the posterior 
  since $\overline{\eta}^\by_\beta$ is Gaussian and can be simulated exactly given access to 
  the second variation $D^2 F$. In situations where this second variation is expensive to compute we propose an 
  alternative approximation to $\eta^\by_\beta$ as follows: 
    \begin{equation}
  \label{eqn: sampling posteriors density Gauss-Newton-Laplace approximation}
    \begin{aligned}
    \frac{\dd \widetilde{\eta}^\by_\beta}{\dd \Lambda}(\bz)  \propto
    \exp \left( - \frac{1}{2\beta^2} | F(\bz^\by_\beta) + 
    \nabla F(\bz^\by_\beta)^T (\bz-\bz^\by_\beta) - \by |^2  - \frac{1}{2} \bz^T \Theta^{-1} \bz \right)\, .
  \end{aligned}
  \end{equation}
  We refer to this measure as the Gauss-Newton approximation to $\eta^\by_\beta$ as it 
  arises from the probabilistic interpretation of the Gauss-Newton algorithm of \cite{chen2021solving}
  that was proposed for finding the mode $\bz^\by_\beta$. The advantage of the Gauss-Newton 
  approximation over the regular Laplace approximation is that it only uses $\nabla F$ and not its second 
  variation,
  
  The Laplace and Gauss-Newton approximations are related to each other, indeed we have 
  \begin{equation*}
  \frac{\dd \overline{\eta}^\by_\beta}{\dd \Lambda}(\bz)  \propto
      \frac{\dd \tilde{\eta}^\by_\beta}{\dd \Lambda}(\bz) \exp \left( -\frac{1}{2\beta^2}(\bz-\bz^\by_\beta)^T
      \left[ D^2 F(\bz^\by_\beta) (F(\bz^\by_\beta)-\by) \right] (\bz-\bz^\by_\beta) \right)
    \, ,
  \end{equation*}
 implying that the Gauss-Newton approximation is close to Laplace whenever $F(\bz^\by_\beta) - \by$
 is small. We anticipate that this approximation is accurate in the regimes where 
  density $\eta^\by_\beta$ would concentrate around the set 
 $F^{-1}(\by)$. Our numerical experiments indicate that 
 this happens in the GP-PDE setting when we have a lot of observation points 
 and $\beta \to 0$, however, we do not expect this approximation to be good in the 
 setting where $\beta \to 0$, but only a few observations are available.

\subsection{Sampling Strategies for Conditional Measures ($\beta^2 = 0$)}
\label{ssec:2}
The conditional measure $\mu^\by_0$ can be simulated using 
similar ideas from the previous section. By \Cref{prop:conditional-measure-decomposition}, we can write $\mu^\by_0 = \mu^\bphi \ast \bchi^T_\sharp \eta^\by_0$. 
Once again, the measure $\mu^\bphi$ can be simulated (up to discretization errors) exactly and so it remains to generate samples from
 $\eta^\by_0= \eta( \dd \bz | F(\bz) = \by)$, the conditional of
   $\eta = N( 0, \Theta)$ with respect to the map $F$. To do so, 
   we will identify an explicit expression for the 
   Lebesgue density of this conditional. For simplicity we assume that there is a decomposition 
$\bz = (\bz_1, \bz_2)$ 
such that $F(\bz)-\by = 0$ is equivalent to $\bz_2-G(\bz_1; \by)=0$ for some mapping $G$ depending on $\by$ (and implicitely $F$). Here $\bz_1 \in \R^{N_1}, \bz_2 \in \R^{N_2}$ such that $N=N_1+N_2$. Such a decomposition is often easy to obtain 
in many practical applications including the GP-PDE 
example of \Cref{subsec:motivation} and can generally be guaranteed by the implicit function theorem under mild 
conditions on $F$. 



With this decomposition, and a slight abuse of notation, we have $$\eta^\by_0= \eta( \dd \bz | \bz_2=G(\bz_1; \by)) = \eta( \dd \bz_1 | \bz_2=G(\bz_1; \by)) \updelta_{G({\bz_1; \by})}(\bz_2).$$ In the following proposition, we identify the formula for $\eta( \dd \bz_1 | \bz_2=G(\bz_1))$ by taking the limit of $\beta \to 0$ for $\eta({\rm d}\bz_1|\bz_2=G(\bz_1)
    +\beta\xi)$ where $\xi \sim N(0,I)$ is the centered Gaussian distribution with identity covariance in $\R^N$.

\begin{proposition}
\label{prop: density of conditionals}
    Let $\bz = (\bz_1, \bz_2) \sim N(0,\Theta)$, where $\bz_1 \in \R^{N_1}, \bz_2 \in \R^{N_2}$, and $\Theta$ is non-singular. Consider the measure  
    $\breve{\eta}_\beta({\rm d}\bz_1) := \eta({\rm d}\bz_1|\bz_2=G(\bz_1)+
    \beta\xi)$ where $\xi \sim N(0,I)$ and $G$ is a measurable function\footnote{Note that the dependence of $G$ on $\by$ is suppressed here since the theorem holds for arbitrary measurable maps $G$} in $\R^{N_1}$. Then 
    the density of $\breve{\eta}_\beta$ converges uniformly as $\beta \to 0$ to a density  $\breve{\eta}_0$, where
    \begin{equation}
    \label{eqn: conditional measure formula}
        \frac{\dd \breve{\eta}_0}{\dd \Lambda}(\bz_1) \propto \exp\left(-\frac{1}{2}\left(\bz_1,G(\bz_1)\right)\Theta^{-1} \begin{pmatrix}\bz_1 \\ G(\bz_1)\end{pmatrix}\right) \, .
    \end{equation}
\end{proposition}

\begin{proof}
    We can write down the density of $\breve{\eta}_\beta$ using Bayes' formula:
    \begin{equation}
    \begin{aligned}
       & \breve{\eta}_\beta(\dd \bz_1) \propto  \Lambda( \dd \bz_1) \int \exp\left(-\frac{1}{2}\left(\bz_1,\bz_2\right)\Theta^{-1}\begin{pmatrix}\bz_1 \\ \bz_2\end{pmatrix}\right) \exp\left(-\frac{|\bz_2-G(\bz_1)|^2}{2\beta^2}\right) \Lambda( \dd \bz_2)\\
        &\propto \Lambda( \dd \bz_1) \int \exp\left(-\frac{1}{2}\left(\bz_1, G(\bz_1)+\bz_3\right)\Theta^{-1} \begin{pmatrix}\bz_1 \\ G(\bz_1)+\bz_3\end{pmatrix}\right) \exp\left(-\frac{|\bz_3|^2}{2\beta^2}\right) \Lambda( \dd \bz_3)\, ,
    \end{aligned}
    \end{equation}
   where we have used the change of variables $\bz_2 = G(\bz_1) + \bz_3$. Let us define 
   \[g(\bz_1, \bz_3) := \exp\left(-\frac{1}{2}\left(\bz_1,G(\bz_1)+\bz_3\right)\Theta^{-1} \begin{pmatrix}\bz_1 \\ G(\bz_1)+\bz_3\end{pmatrix}\right) \, , \] so that we can write
 $\breve{\eta}_{\beta}({\rm d}\bz_1) \propto \Lambda( \dd \bz_1) \int  g(\bz_1,\bz_3) \rho_{\beta}(\bz_3)  \Lambda(\dd \bz_3)$  where $\rho_{\beta}$ is the density of a Gaussian random variable with mean $0$ and covariance $\beta^2I$. As $\rho_{\beta}$ is a mollifier, it holds that 
       $\lim_{\beta \to 0} \int  g(\bz_1,\bz_3) \rho_{\beta}(\bz_3) \Lambda( \dd \bz_3) = g(\bz_1,0)$ for any $\bz_1$; here such convergence is also uniform for all $\bz_1$    which 
       yields the uniform convergence of the density of $\breve{\eta}_{\beta}$ to that of $\breve{\eta}_0$, as $\beta \to 0$ as desired.
       
       To verify the claimed uniform convergence above, consider
   \begin{equation*}
       \begin{aligned}
           \Bigl|\int  g(\bz_1,\bz_3) \rho_{\beta}(\bz_3) 
           \Lambda(\dd \bz_3) - g(\bz_1,0)\Bigr| \leq & \int  \bigl|g(\bz_1,\bz_3)-g(\bz_1,0)\bigr| \rho_{\beta}(\bz_3) \Lambda( \dd \bz_3) \\\leq &\sup_{\bz_1,\bz_3} \bigl|\nabla_{\bz_3} g(\bz_1,\bz_3)\bigr| \int |\bz_3| \rho_{\beta}(\bz_3) \Lambda( \dd \bz_3) 
           \leq C \beta\, ,
       \end{aligned}
   \end{equation*}
   where $C$ is a universal constant 
   that depends only on the dimension $N$ and the eigenvalues of $\Theta$, but 
   independent of $\bz_1,\bz_3$ and $\beta$ since
   \begin{equation*}
       \begin{aligned}
           \sup_{\bz_1,\bz_3} \bigl|\nabla_{\bz_3} g(\bz_1,\bz_3)|&\leq \sup_{\bz_1,\bz_3} \Bigl|\Theta^{-1} \begin{pmatrix}\bz_1 \\ G(\bz_1)+\bz_3
\end{pmatrix}\Bigr|\cdot |g(\bz_1,\bz_3)|  \\ 
& \leq \sup_{\bx} |\bx| \exp\Bigl(-\frac{1}{2}\bx^T\Theta \bx\Bigr) \leq C_1.
       \end{aligned}
   \end{equation*}
   Here $C_1 \ge 0 $ is a constant that also  depends on $N$ 
   and the spectrum of $\Theta$.
   Moreover, by the standard  moment formula for Gaussian distributions, it holds that $\int |z_3| \rho_{\beta}(\bz_3) \Lambda( \dd \bz_3) \leq C_2\beta$. Taking $C=C_1C_2 \ge 0$ leads to the desired result. 
{}
\end{proof}

Since \Cref{prop: density of conditionals} gives a closed form  expression for the Lebesgue density of the conditional 
measure $\eta_0$, we can use standard algorithms, such 
as those discussed in \Cref{ssec:1}, to (approximately) sample this measure. Notably, letting $\bz_1^\dagger$ denote the  mode of $\eta_0$,  the Gauss-Newton approximation to \eqref{eqn: conditional measure formula} will now correspond to the measure
\begin{equation}
   \begin{aligned}
    &\frac{\dd \tilde{\eta}_0}{\dd \Lambda} (\bz_1)\\
    & \propto \exp\left(-\frac{1}{2}\left(\bz_1,G(\bz_1^\dagger) + \nabla G(\bz_1^\dagger)(\bz_1 - \bz_1^\dagger) \right)\Theta^{-1} \begin{pmatrix}\bz_1 \\ G(\bz_1^\dagger) + \nabla G(\bz_1^\dagger) (\bz_1 - \bz_1^\dagger)\end{pmatrix}\right)\,.
   \end{aligned}
\end{equation}

\subsection{Numerical Experiments} \label{sec:numerics}
Our numerical experiments contain two parts: The first part investigates 
the Laplace and Gauss-Newton approximations introduced 
in \Cref{ssec:1,ssec:2}, for (approximately) 
sampling the posterior and conditional distributions. 
The second part applies our methodology to GP-PDE solvers for 
example nonlinear PDEs.
In \Cref{sec:numerics-Demonstration of Consistency}, 
we compare, through numerical experiments, MCMC, the 
Laplace approximation and its Gauss-variant; we show that, on the examples considered,
the Laplace and Gauss-Newton approximations are good approximations to MCMC in certain regimes as the posteriors 
concentrate around the true values of the parameter, making Gauss-Newton a good 
approximation to Laplace.
We apply our methodology to perform UQ as a proxy for error estimation for 
GP-PDE solvers in \Cref{sec:numerics-Using Posterior for Error Estimates}. In \Cref{sec:numerics-Using Posterior for Sampling Informative Points}, we use 
UQ estimates for adaptive selection of collocation points for the solver.

\subsubsection{Laplace vs 
Gauss-Newton}
\label{sec:numerics-Demonstration of Consistency}
In this subsection, we numerically demonstrate, in a nonlinear elliptic PDE example, the accuracy of Laplace and Gauss-Newton approximations  when compared to (the viewed as gold standard) MCMC algorithms.
We consider the PDE \eqref{elliptic-PDE}  with $d=2$ and $\tau(u) = 10 u^3$ 
and choose the ground truth solution $u^\dagger(\bx) = \sin (\pi \bx_1)\sin(\pi \bx_2) + \sin(3\pi \bx_1)\sin(3\pi \bx_2)$ and determine the right-hand 
side $f$ which gives this solution, noticing that the Dirichlet boundary conditions are readily satisfied 
by the prescribed solution. We take $J$ collocation points on a uniform 
grid  in the interior of the domain and $M$ uniform  points on the boundary.
We denote the interior points by $\bx_1,...,\bx_J$ and the boundary points by $\bx_{J+1},...,\bx_{M}$. 
For our experiments we took 
$(J,M)= \{ (16,25), (49, 64), (81,100) \}$.
Following \Cref{subsec:motivation}, we then define $F(\bphi(u))$ and $\by$, 
based on these collocation points and on $f$, such that identity $F(\bphi(u)) = \by$ encodes the PDE constraint at the collocation points. 

Suppose $u$ is a priori distributed according to the GP $\mu = N(0,\mathcal{K})$ where $\mathcal{K}$ is the integral operator corresponding to the Mat\'ern kernel with regularity parameter $\nu = 7/2$ \cite[Sec.~4.2.1]{williams2007gaussian}. Then the conditional $\mu_0^{\by}$ encodes information
about  the solution to the PDE. We compute the conditional mode $u_0^{\by}$ using the Gauss-Newton optimization algorithm of \cite{chen2021solving}. 
Using this  mode we further compute the Laplace and Gauss-Newton approximations to 
$\mu^\by_0$ following the approach of \Cref{ssec:2}. 

In \Cref{fig-mcmc-laplace} (top row), we compare the true solution of the PDE to the
MAP estimator $u^\by_0$ and the
posterior mean of the MCMC samples with $(J, M) = (81, 100)$. We observe that the MAP and the MCMC mean 
are comparable approximations to the true solution, indicating that the posterior measure 
is concentrated around the truth.
This claim is further supported by \Cref{fig-mcmc-laplace} (bottom row) where we compare the 
pointwise standard deviations computed by MCMC, Gauss-Newton, and Laplace. We see good agreement 
between all three methods, suggesting that (a) the posterior is close to being Gaussian and (b) 
the Gauss-Newton approximation is as good as Laplace. 
In 
 \Cref{table-error-of-Gaussian-approx} we further compare the relative $L^2$ error 
 between the MCMC mean and standard deviations with those of Laplace and Gauss-Newton approximations.
 We observe that 
 not only does the MAP converge to the MCMC mean but that Laplace and Gauss-Newton approximations 
 to the standard deviation fields converge to that of the MCMC samples. Moreover, the Gauss-Newton and 
 Laplace errors are comparable, with Gauss-Newton achieving  higher errors 
 when collocation points are scarce.

In \Cref{fig-kde-comparison}, we evaluate the posterior fields at the location $x = [0.6, 0.4]$
for different mesh sizes 
and  compare the  kernel density estimator of the MCMC samples to that of Laplace and 
Gauss-Newton approximations. Here we observe that (a) the Laplace and Gauss-Newton 
approximations are very close to each other and (b) as we refine the mesh, these two 
approximations converge to the MCMC posterior. We observed this behavior consistently 
at other locations as well, supporting the claim that the posterior is 
nearly Gaussian around the MAP.

\begin{figure}[htp]
    \centering
    \begin{overpic}[width=0.9\textwidth]{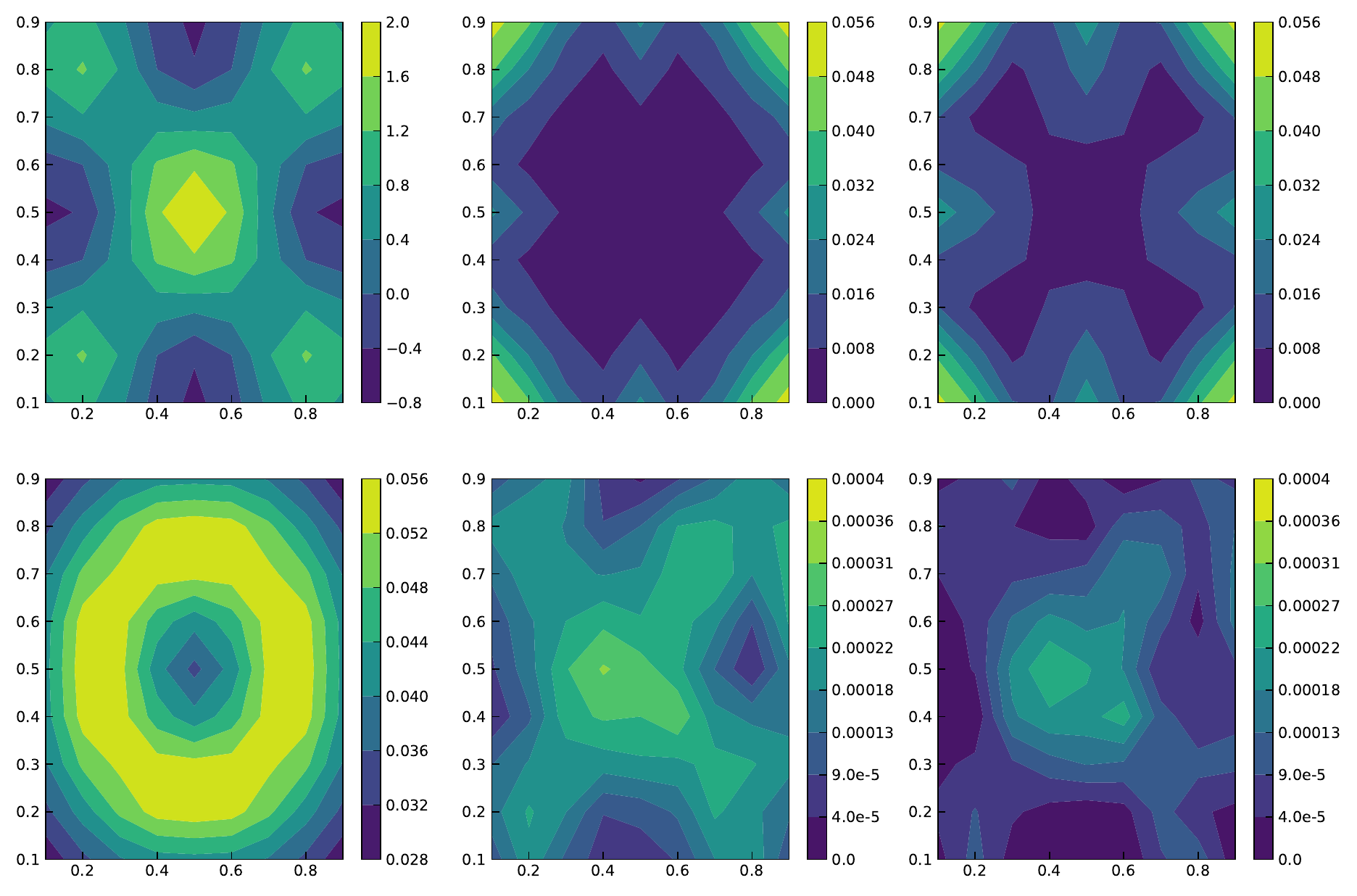}
    \put(7,65){true solution}
    \put(37,65){err: MCMC mean}
    \put(74,65){err: MAP}
    \put(7,32){MCMC std field}
    \put(38,32){err: GN std field}
    \put(68,32){err: Laplace std field}
    \end{overpic}
    \label{fig-mcmc-laplace}
    \caption{Numerical results for nonlinear elliptic \eqref{elliptic-PDE} as described in \Cref{sec:numerics-Demonstration of Consistency} with $(J,M)=(81,100)$ collocation points. Top row: True solution, error of  MCMC mean, and error of the MAP estimator obtained by the GP-PDE methodology. Bottom row: standard deviation field of MCMC samples followed by  its difference from the standard deviation fields obtained using the Gauss-Newton and Laplace approximations.}
    \end{figure}

\begin{figure}[htp]
    \centering
    \begin{overpic}[width=0.31 \textwidth]{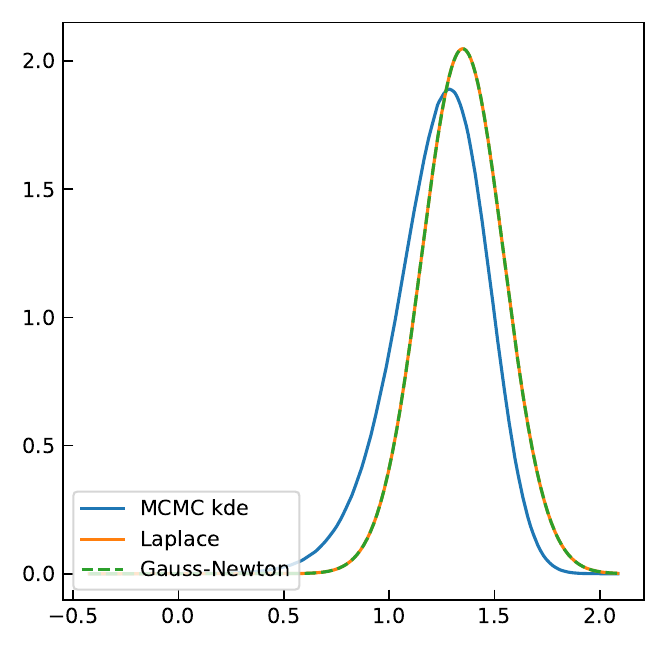}
        \put(25,98){\footnotesize $(M,J) = (16,25)$}
    \end{overpic}
    \begin{overpic}[width=0.3 \textwidth]{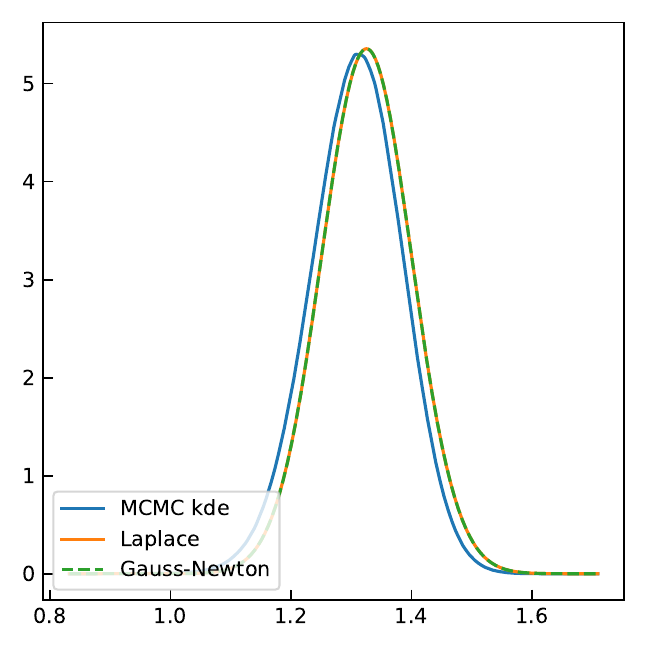}
        \put(25,100){\footnotesize $(M,J) = (49,64)$}
    \end{overpic}
    \begin{overpic}[width=0.3 \textwidth]{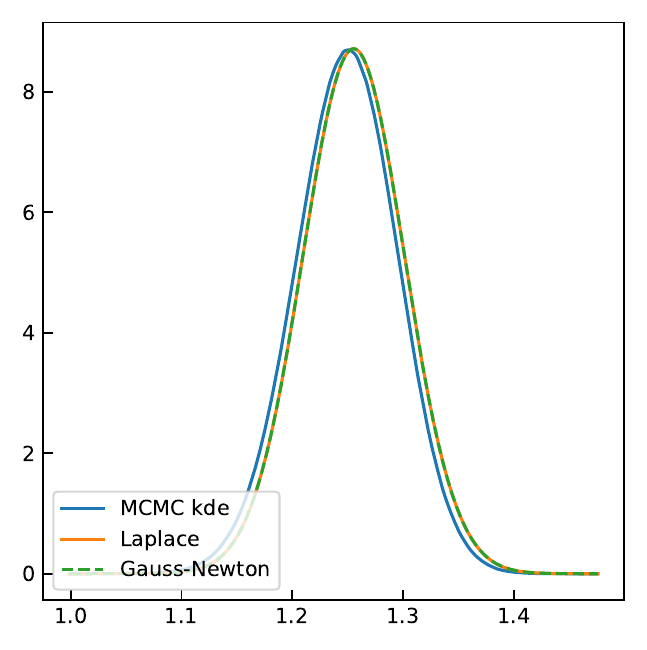}
        \put(25,100){\footnotesize $(M,J) = (81,100)$}
    \end{overpic}
    \label{fig-kde-comparison}
    \caption{Pointwise numerical results for the nonlinear elliptic PDE \eqref{elliptic-PDE} as described in \Cref{sec:numerics-Demonstration of Consistency}. Here we compared the conditional distribution of the solution to its various  approximations at a single point $[0.6,0.4]$ with (Left) $(M,J)=(16,25)$, (middle) $(M,J)=(49,64)$, and (right) $(M,J) =(81,100)$ 
    collocation points.}
    \end{figure}

\begin{table}[htp]
\label{table-error-of-Gaussian-approx}
\footnotesize
\centering
\begin{tabular}{cccc}
\hline
 Relative $L^2$ error &  $(J,M)= (16,25)$  &$(J, M) = (49, 64)$ & $(J, M) = (81,100)$  \\ \hline
MAP vs MCMC mean& 1.086e-1 & 1.682e-2 & 6.320e-3 \\
(std) Laplace vs MCMC & 6.360e-2 &5.557e-3 & 2.136e-3 \\
(std) Gauss-Newton vs MCMC & 7.934e-2 & 1.038e-2 & 4.086e-3 \\ \hline
\end{tabular}
\caption{Relative $L^2$ error for the mean and stanfard deviation of the Laplace approximation and its Gauss-Newton variant at sampled points compared with MCMC for the nonlinear 
elliptic PDE \eqref{elliptic-PDE} as described in \Cref{sec:numerics-Demonstration of Consistency}.}
\end{table}

\subsubsection{UQ for GP-PDE}
\label{sec:numerics-Using Posterior for Error Estimates}
One of the advantages of the GP-PDE perspective is that the conditional/posterior uncertainties 
can be readily computed as a priori indicators of the performance of the algorithm.
Here we will investigate the usefulness of such uncertainties in the context of 
our nonlinear elliptic PDE \eqref{elliptic-PDE} as well as Burgers' equation.

\paragraph{Nonlinear Elliptic PDE}
We  start by considering the nonlinear elliptic PDE \eqref{elliptic-PDE} once more with 
$\tau(u) = \alpha u^3$ along with prescribed solution $u^\dagger(\bx) = \sin(\bx_1)\sin(\bx_2) + \sin(10\bx_1)\sin(a\bx_2)$ with scalar parameters $\alpha, a >0$ to be chosen later.  
We solve the PDE using $(J, M) = (16, 25)$ with the prior $\mu = N(0, \mathcal{K})$
with $\mathcal{K}$ being the $7/2$-Mat\'ern kernel. To estimate the conditional mode and 
standard deviations we ran three steps of the Gauss-Newton algorithm for different choices of $(\alpha, a)$ as shown in \Cref{fig-UQ-band}. We observe that in the linear PDE setting where $\alpha= 0$, 
the resulting posterior standard deviation field is very smooth and is known to be 
independent of the PDE solution and only dependent on the collocation points. As expected, 
maximum standard deviation occurs in the middle of the domain as is often expected in GP regression. 
Interestingly, 
the posterior standard deviation fields appear to change noticeably with stronger nonlinearities. 
In particular, the maximum uncertainty no longer occurs in the middle of the domain but rather 
over a non-trivial set. 

It is well-known, in the context of GP regression \cite[Thm.5.1]{Owhadi:2014} 
that if $u^\dagger$ is the ground truth and $u^\by_0$ is its GP interpolant, that the 
following error bound holds  
\begin{equation}\label{uppser-lower-bound}
    |u^\dagger(\bx)-u^\by_0(\bx)| \leq \|u^\dagger\|_{\mH(\mu)} \sigma(\bx)\qquad \forall \bx \in \Omega,
\end{equation}
where $\sigma(\bx)$ is the standard deviation field of the conditioned GP and $\| \cdot \|_{\mH(\mu)}$
denotes the Cameron-Martin/RKHS norm of $u^\dagger$ corresponding to the GP prior $\mu$. It is 
therefore natural to investigate, numerically, whether this error bound remains valid in the 
case of the GP-PDE solver. Since in practice we do not have access to $\| u^\dagger \|_{\mH(\mu)}$,
we replace it with the Cameron-Martin norm of the MAP, i.e., $\| u^\by_0 \|_{\mH(\mu)}$.

In \Cref{fig-UQ-slice-elliptic-PDE} we show a slice of the PDE solution $u^\dagger$ along with 
the GP-PDE solution and the requisite error bounds computed 
using the standard deviation fields for our nonlinear elliptic PDE example. We observe that 
in all three cases, the conditional mode $u^\by_0$ is a good approximation to $u^\dagger$
while the upper and lower bounds computed via \eqref{uppser-lower-bound} always contain both 
the numerical and true solutions. However, we note that the computed error bands appear to 
be too large compared to the actual error of the numerical solution.


\begin{figure}[htp]
    \centering
    \begin{overpic}[width=0.32 \textwidth]{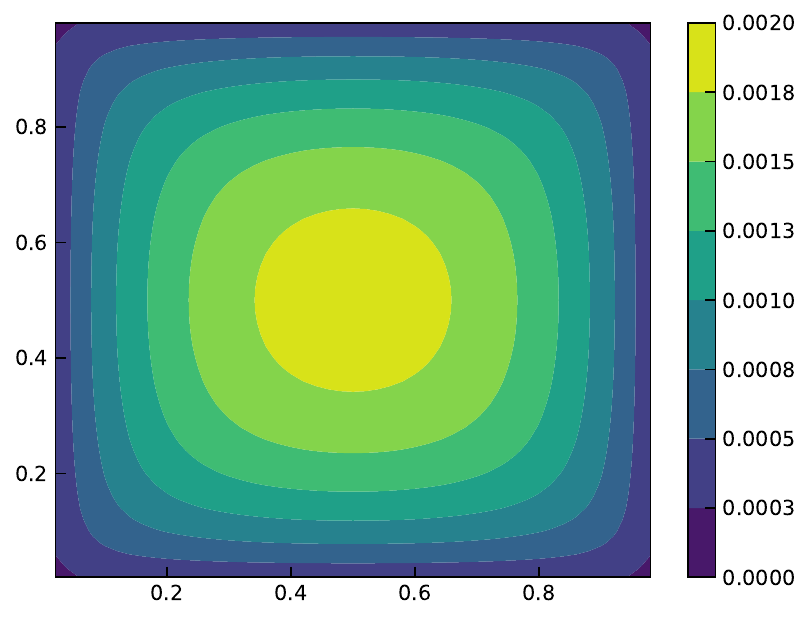}
        \put(37, 77){\footnotesize $\alpha = 0$}
    \end{overpic}
        \begin{overpic}[width=0.32 \textwidth]{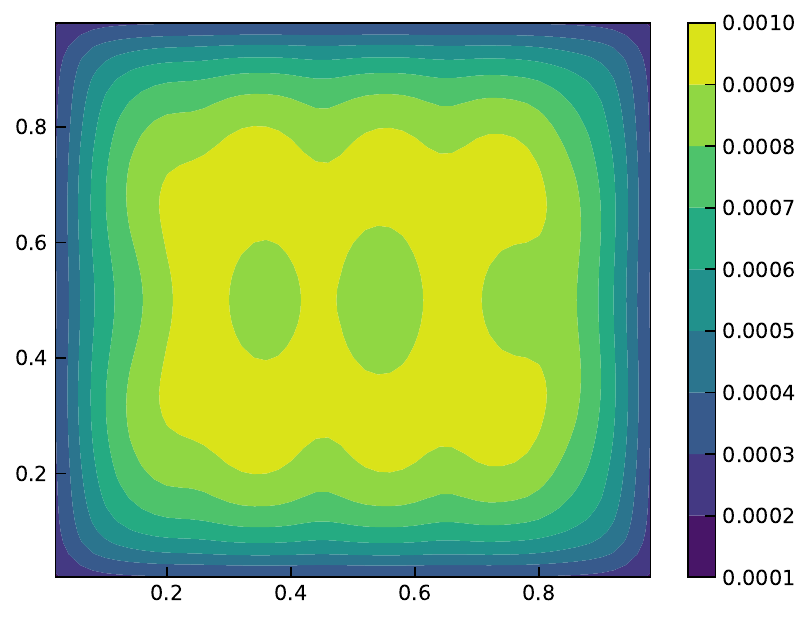}
        \put(24, 77){\footnotesize $(\alpha, a) = (10, 3)$}
    \end{overpic}
        \begin{overpic}[width=0.32 \textwidth]{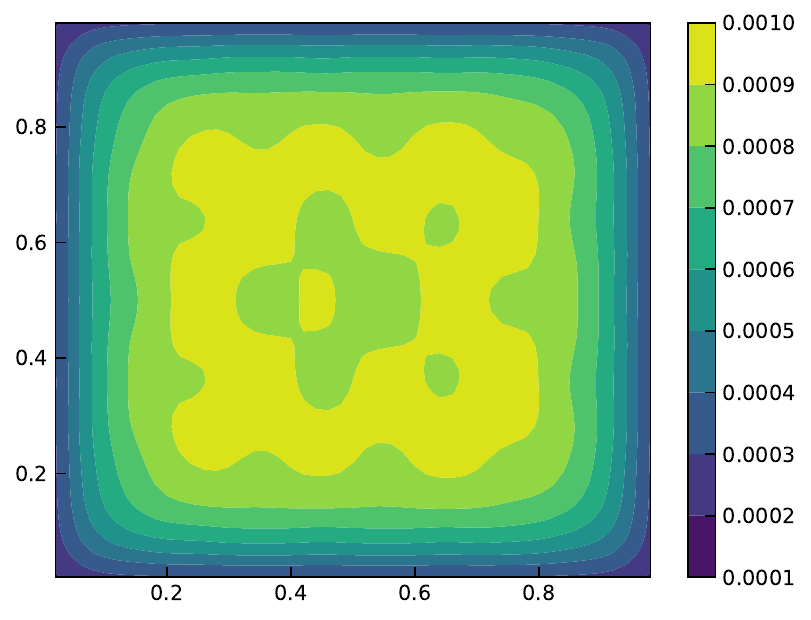}
        \put(24, 77){\footnotesize $(\alpha, a)  = (10, 7)$}
    \end{overpic}    
    \label{fig-UQ-band}
    \caption{Comparing posterior standard deviation fields for 
    the nonlinear elliptic PDE \eqref{elliptic-PDE} as described in  \Cref{sec:numerics-Using Posterior for Error Estimates}. From left to right 
    the panels show the standard deviation fields for increasingly stronger nonlinearities.}
    \end{figure}

\begin{figure}[htp]
    \centering
    \begin{overpic}[width=.32 \textwidth]{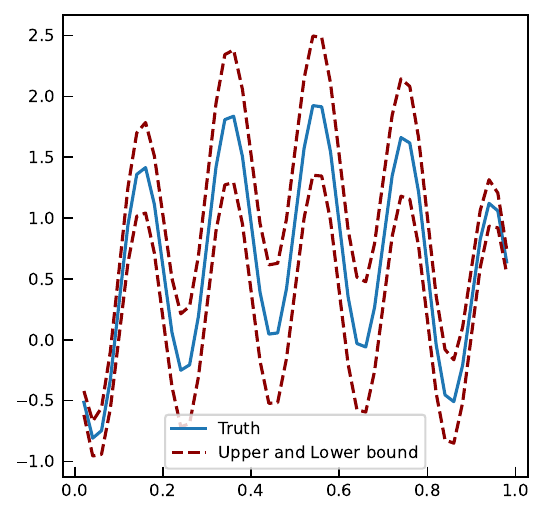}
        \put(35,95){\footnotesize $(\alpha,a) = (0,3)$}
    \end{overpic}
    \begin{overpic}[width=.32 \textwidth]{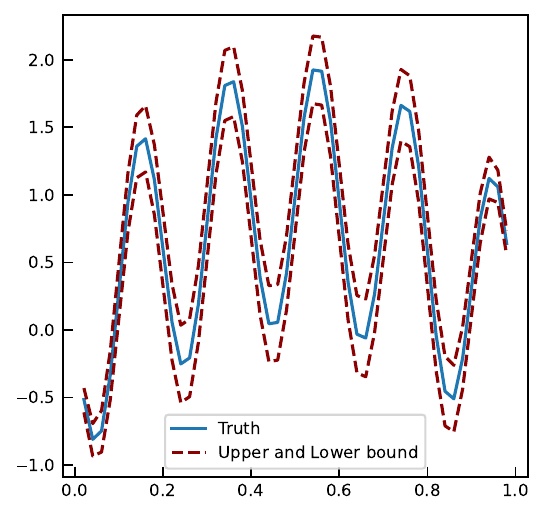}
        \put(35,95){\footnotesize $(\alpha, a) = (10, 3)$}
    \end{overpic}
    \begin{overpic}[width=.32 \textwidth]{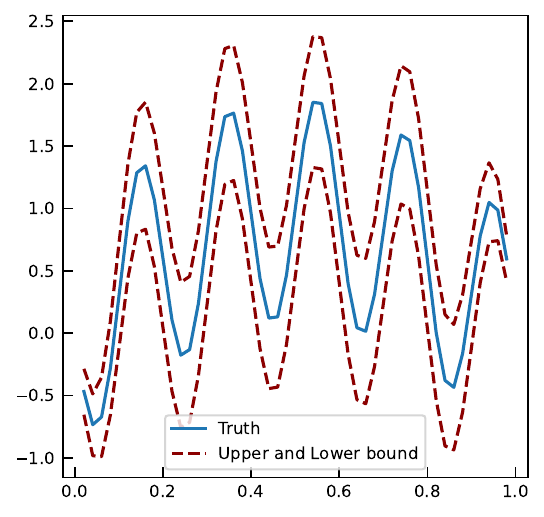}
        \put(35,95){\footnotesize $(\alpha, a) = (10, 7)$}
    \end{overpic}
    \label{fig-UQ-slice-elliptic-PDE}
    \caption{Truth and the upper and lower error bound obtained by the GP-PDE method, for the slice $\bx_2=0.5$, in the nonlinear elliptic PDE \eqref{elliptic-PDE} as described in \Cref{sec:numerics-Demonstration of Consistency}. From left to right 
    the panels show the posterior mean with uncertainty bands for increasingly stronger
    nonlinearities.}
    \end{figure}

\paragraph{Burgers' Equation}
Next we consider the viscous Burgers equation:
\begin{equation}
    \label{Burgers-proto-PDE}
    \begin{aligned}
      \partial_t u +  u \partial_x u  - 0.01  \partial_x^2 u  &= 0, \quad \forall
      (x,t) \in (-1, 1)  \times (0,1]\, , \\
      u(x, 0) & = - \sin( \pi x)\, , \\
      u(-1, t) & = u(1, t)  = 0\, .
  \end{aligned}
  \end{equation}
We solved this equation using the space-time GP-PDE approach of  \cite{chen2021solving}. 
Collocation points were uniformly distributed on a regular grid with time step size ${\rm d}t = 0.05$ and
spatial step size ${\rm d}x=0.0125$. {The kernel of the covariance function of the GP is chosen as the anisotropic Gaussian kernel, same as \cite{chen2021solving}:}
\begin{equation}
    K\Bigl((x,t),(x',t'); \sigma\Bigr) = \exp\Bigl(-\sigma_1^{-2}(x-x')^2-\sigma_2^{-2}(t-t')^2\Bigr)\, 
\end{equation}
 with  $\sigma = (1/20,1/3)$.
We ran $15$ steps of Gauss-Newton to obtain the conditional mode and the corresponding  approximation to the conditional covariance matrix. In \Cref{fig-burgers_contour_adaptive} (left and middle)
we show the GP-PDE solution to the Burgers' equation as well as the posterior standard deviation 
estimated using Gauss-Newton. We clearly observe that the standard deviation is peaked around the 
location of the (near) discontinuity in the solution, indicating that the standard deviation 
field is a good proxy for the adaptive placement of  collocation points. 

\begin{figure}[htp]
    \centering
    \begin{overpic}[width=0.32\textwidth]{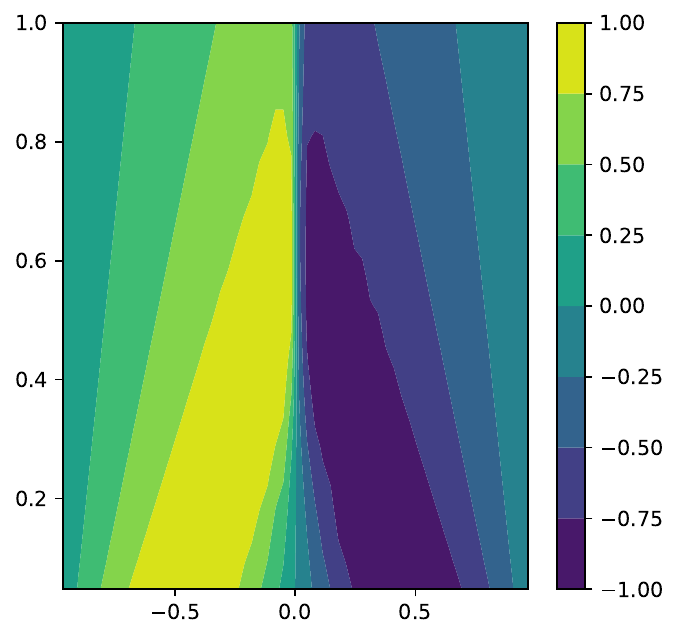}
    \put(27,95){\footnotesize true solution}
    \end{overpic}
    \begin{overpic}[width=0.32\textwidth]{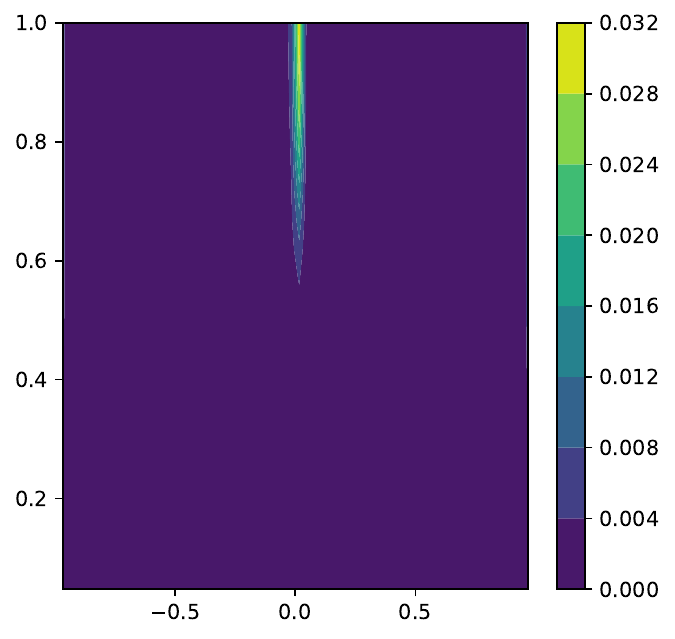}
    \put(20,95){\footnotesize Gauss-Newton std}
    \end{overpic}
    \begin{overpic}[width=0.32\textwidth]{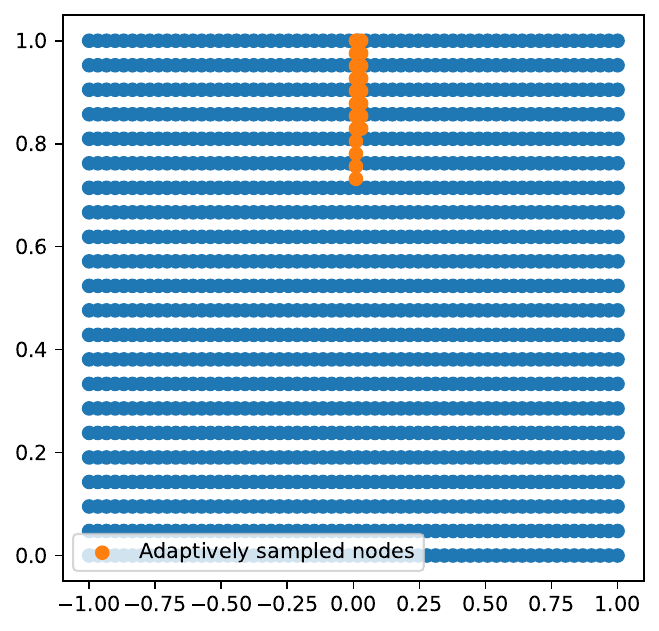}
    \put(26,96){\footnotesize adaptive sampling}
    \end{overpic}
    
    \label{fig-burgers_contour_adaptive}
    \caption{Numerical experiments for the Burgers' PDE \eqref{Burgers-proto-PDE}. Left: Contour plot of the MAP estimator of the solution in space-time; Middle: Contour plot of the conditional standard deviation field; Right: Adaptively sampled collocation points 
    guided by areas of concentrated uncertainty.}
    \end{figure}


\subsubsection{Adapting Collocation Points}
\label{sec:numerics-Using Posterior for Sampling Informative Points}
Based on our observation in the previous section (e.g. 
\Cref{fig-burgers_contour_adaptive}) it is of interest to 
investigate whether the UQ estimates from the posterior/conditional 
measure can be used for the adaptation of collocation points for PDE solvers.
For example, we may add more collocation points in areas of 
maximum variance under the posterior/conditional on the solution of 
the PDE.


For our first experiment we considered the Burgers equation \eqref{Burgers-proto-PDE} which was originally solved on a uniform grid 
and added 30 new collocation points in the region of  maximum posterior
variance which happens to surround the (smoothed) shock. This produces a non-uniform grid of collocation 
points as shown on the right panel of \Cref{fig-burgers_contour_adaptive}. In our experiments we observed that 
adding these new points leads to a factor $2$ improvement in
the $L^\infty$ error of the solution at time $t=1$.
This demonstrates the effectiveness of using UQ estimates for adaptive selection of collocation points. We observed that when we continued to select points based on this greedy approach, the improvement in accuracy was less significant and sometimes even numerical instability occurs. We attribute this phenomenon to the use of a global space-time formulation, which overlooks the causality of time dependent PDEs and could lead to numerical challenges. This could also be attributed to the ill-conditioning of the involved kernel matrices associated to a large number of points packed in a small region of the domain which further 
warrants the use of a nugget term.

For our second experiment we return to the 
nonlinear elliptic PDE \eqref{elliptic-PDE} with 
$\tau(x) = 10 x^3$. We prescribe the exact solution 
$u(\bx) = 2^{4p}  \bx_1^{2p}(1-\bx_1)^p\bx_2^{2p}(1-\bx_2)^p$
 with  $p=10$ as shown in \Cref{fig-sample-points}; this example is designed to have a highly localized feature around the location $(2/3,2/3)$. We then solve the PDE and adaptively add collocation points 
 as follows: 
 (1) Start with $100$ uniformly sampled collocation points in the interior and on the boundary of the unit box;
 (2) compute the Gauss-Newton approximation to the posterior of the 
 solution and sample $50$ new collocation points in areas of largest 
 posterior variance; (3) repeat step (2) for 10 iterations 
 to get a total of $600$ collocation points in the interior.

In the bottom left panel of \Cref{fig-sample-points}  we show an instance of 
the collocation points obtained by the above procedure which 
may be compared with the top right panel, depicting a uniform 
set of collocation points. We see that the posterior 
adapted points are blind to the concentrated features of 
the solution to the PDE, contrary to our early example for Burgers' 
equation. We further modified our adaptive 
sampling of the collocation points to place new points in regions of 
large equation residual which produced the bottom right panel of 
\Cref{fig-sample-points}. We observe that this new strategy leads to 
collocation points that are clustered around the main feature of 
the solution. We present $L^2$ and $L^\infty$ errors of the 
solutions obtained by the three sampling strategies in \Cref{table-sample-points}, showing that the conditional variance adaptation scheme 
leads to an order of magnitude improvement in the error over uniform 
points while residual adaptation leads to yet another order of magnitude improvement. 

\begin{figure}[htp]
    \centering    
    \includegraphics[width=0.45\textwidth]{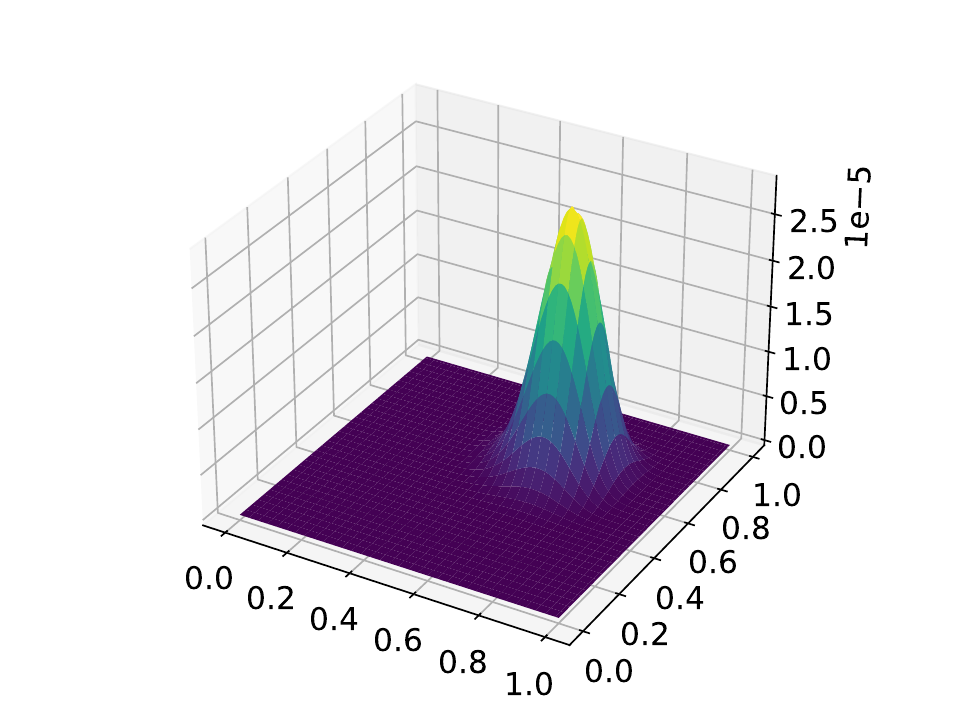}
    \includegraphics[width=0.45\textwidth]{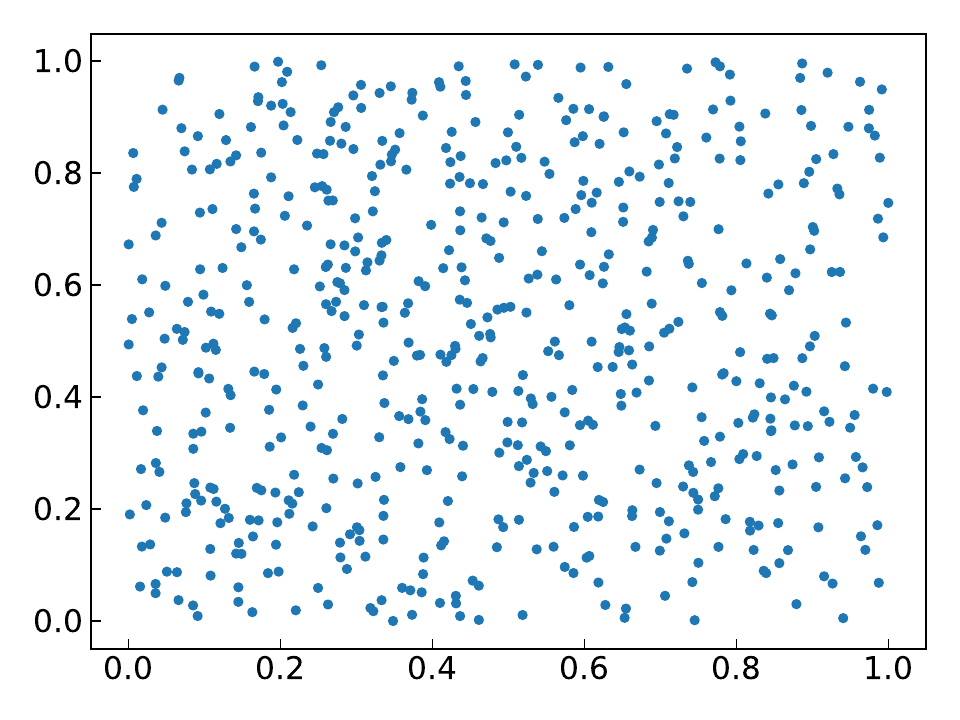}
    \\ \includegraphics[width=0.45\textwidth]{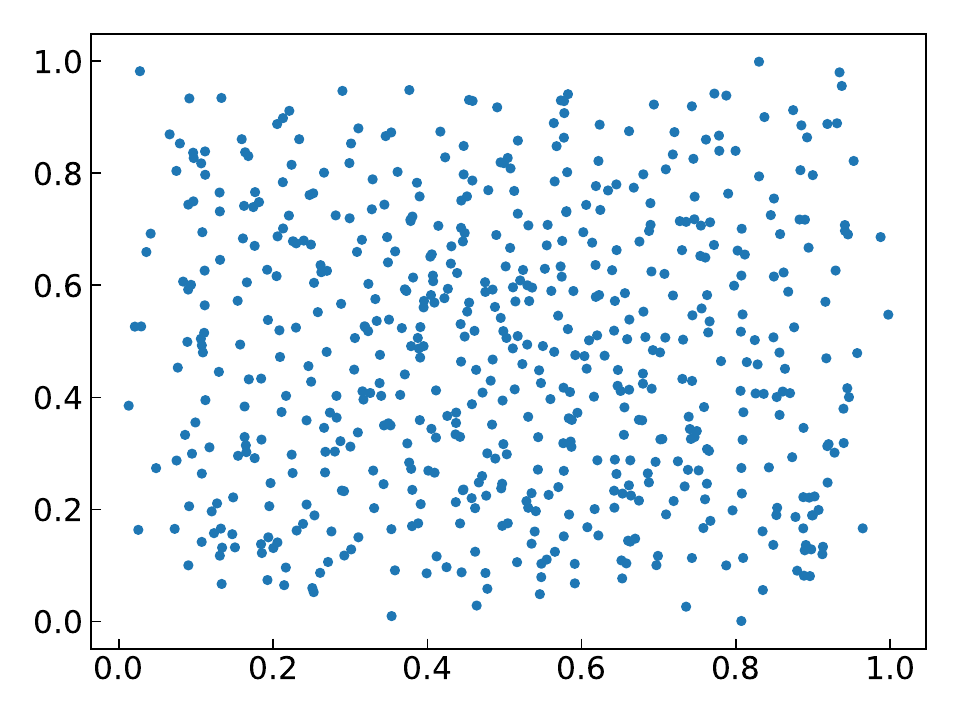} \includegraphics[width=0.45\textwidth]{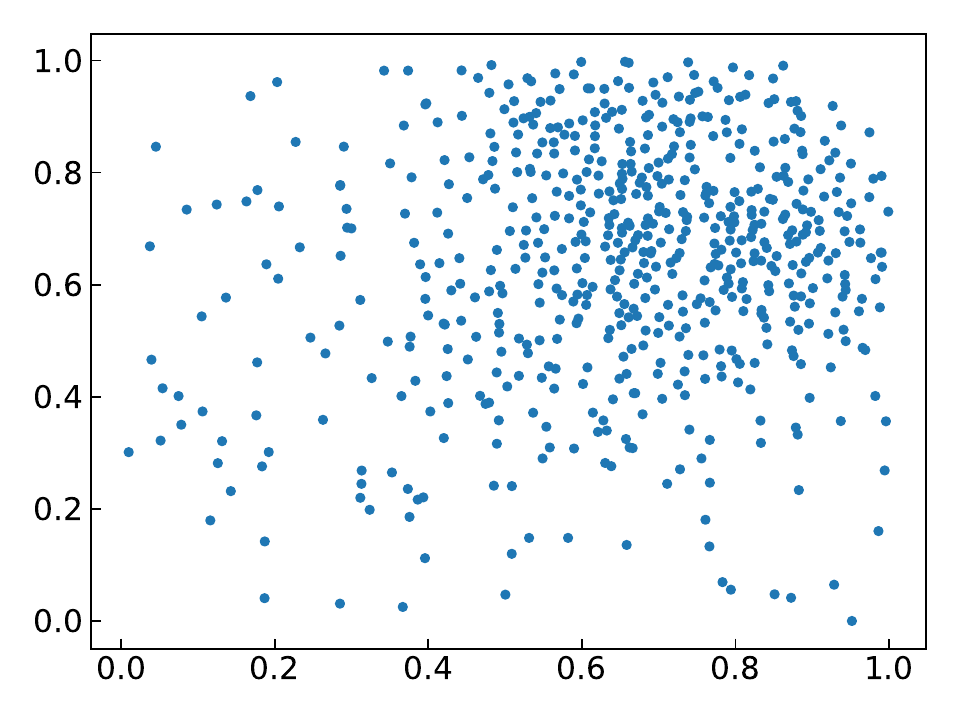}
    \label{fig-sample-points}
    \caption{
    An instance of adaptively selected collocation points for the nonlinear 
    elliptic PDE \eqref{elliptic-PDE} as described in \Cref{sec:numerics-Using Posterior for Sampling Informative Points}.
    Top left: true solution; top right: uniform sampling; bottom left: greedy sampling based on the conditional standard deviation; bottom right: greedy sampling based on equation residues.}
    \end{figure}

    \begin{table}[htp]
\centering
\begin{tabular}{cccc}
\hline
 Sampling strategy & Uniform  & Conditional variance  & Equation residue \\ \hline
Relative $L^2$ error        & 2.337e-2         & 3.345e-3                          & 1.365e-4                        \\
Relative $L^{\infty}$ error & 1.565e-2         & 2.554e-3                          & 1.046e-4                        \\ \hline
\end{tabular}
\caption{Relative $L^2$ and $L^{\infty}$ errors of the numerical solutions 
of the nonlinear elliptic PDE \eqref{elliptic-PDE} as described in \Cref{sec:numerics-Using Posterior for Sampling Informative Points}. Various strategies for adaptive 
sampling of collocation points were compared.  The errors were averaged over $20$ trials.}
\label{table-sample-points}
\end{table}

These experiments show the advantages and potential limitations of using the posterior/conditional variance for adapting collocation points. Interestingly, in the case of Burgers' equation the conditional 
variance captures the interesting structures in the solution while 
this property is not prominent in the case of our nonlinear elliptic PDE.


\section{Conclusions}\label{sec:discussion}
Our focus in this article was the characterization of 
Gaussian measures conditioned on finite nonlinear observations that 
are obtained as the composition of a nonlinear map with a bounded 
and linear operator. We showed that (1) such conditionals 
can be characterized as the limit of posterior measures with 
noisy observations with vanishing small noise standard deviation. We 
showed that this small-noise limiting argument also applied to the 
MAP estimators of the resulting conditionals leading to the novel 
definition of a conditional MAP of a Gaussian measure; (2) 
We showed that  the resulting posteriors/conditional measures 
can be decomposed as the convolution of a Gaussian measure that can 
be identified analytically with a finite-dimensional non-Gaussian measure. 
This decomposition mirrored well-known representer theorems from RKHS theory. 
Item (2) further led us to the design of novel algorithms for 
the simulation of Gaussians conditioned on nonlinear observations by 
focusing computational effort on the non-Gaussian component. 

We applied our results to the particular case of the GP-PDE methodology, 
a collocation  method for solving  nonlinear 
PDEs that models the solution of the PDE as a GP conditioned on the 
PDE constraint at the collocation points. We developed two variational 
inference techniques for simulation of the non-Gaussian component in this 
case under the conjecture that, if the collocation points are sufficiently dense then the non-Gaussian component of the posterior should be approximately 
Gaussian around its MAP. Our numerical experiments confirmed this 
claim. We also investigated the usefulness of the resulting 
uncertainty estimates for improving the accuracy of the PDE solver 
by adaptive selection of collocation points. 

While the GP-PDE setting was the main motivation for our work, our results 
have wide application in the field of inverse problems where Gaussian 
priors are widely used in a function space setting. Here one often 
discretizes the problem and samples the posterior using a 
function space MCMC algorithm. However, our results here suggest that 
significant speed up may be achieved by performing MCMC only on 
the non-Gaussian component and directly simulating the Gaussian component, for example by computing the underlying precision matrix of the prior. Our experiments also suggest that this non-Gaussian component may be well approximated by a variational technique such as a Laplace approximation. 
We also observed that our Gauss-Newton approximation (which is only first order) appears to work well in practice, a topic that warrants more detailed theoretical analysis. 

\section*{Acknowledgments}
{All four authors acknowledge support from the Air Force Office of Scientific Research under MURI award number FA9550-20-1-0358 (Machine Learning and Physics-Based Modeling and Simulation). Additionally HO acknowledges support by the Department of Energy under award number DE-SC0023163 (SEA-CROGS: Scalable, Efficient and Accelerated Causal Reasoning Operators, Graphs and Spikes for Earth and Embedded Systems). The work of AMS is also supported by a Department of Defense 
Vannevar Bush Faculty Fellowship.
BH is also supported by the National Science Foundation grant DMS-2208535 (Machine Learning for Bayesian Inverse Problems). YC is also supported by a Courant Instructorship.}

\begin{appendices}

\section{On Optimal Recovery, Game Theory, and Probabilistic Numerics}
\label{appendix:optimal-recovery}

As presented in \cite{owhadi2019operator}, the framework of optimal recovery of Micchelli and Rivlin \cite{micchelli1977survey} provides a natural setting for understanding the correspondence between numerical approximation  and Bayesian inference, which involves the counter-intuitive modeling of a perfectly known function as a sample from a random process.
To describe this consider a Banach space $(\mathcal{B},\|\cdot\|)$ and write $[\cdot,\cdot]$ for the duality product between $\mathcal{B}$ and its dual space $(\mathcal{B}^*,\|\cdot\|_*)$.
When $\mathcal{B}$ is infinite (or high) dimensional, as conceptualized in Information
Based Complexity \cite{Traub1988} (the branch of computational complexity founded on the observation that
numerical implementation requires computation with partial information and limited resources), one cannot directly compute with $u\in \mathcal{B}$ but only with a finite number of \emph{features} of $u$. The type of features we consider here are represented
 as a vector $\Phi(u):=\big([\phi_1,u],\ldots,[\phi_m,u]\big)$
 corresponding to  $m$  linearly independent measurements  $\phi_1,\ldots,\phi_m\in \mathcal{B}^*$. The objective is to  recover/approximate $u$ from  the partial information contained in the feature vector $\Phi(u)$. Then, using the relative error in $\|\cdot\|$-norm as a loss, the classical numerical analysis approach  is to approximate $u$ with the minimizer $v^\dagger$ of
\begin{equation}\label{eqoprec}
\min_{v}\max_{u}\frac{\|u-v(\Phi(u))\|}{\|u\|}\,,
\end{equation}
where the maximum is taken over all $u\in \mathcal{B}$ and the minimum is taken  over all possible functions $v$ of the $m$ linear measurements.
The minimax approximant is (\cite{micchelli1977survey} and \cite[Chap.~18]{owhadi2019operator}) then 
\begin{equation}\label{eqminuhsobfirstdeb}
v^\dagger(y)=\operatorname{argmin}\begin{cases}
\text{Minimize }\|v\|\\
\text{Subject to }v\in \mathcal{B}\text{ and }\Phi(v)=y\,.
\end{cases}
\end{equation}
Furthermore, the minmax problem \eqref{eqoprec} can be viewed as the adversarial zero sum game
in which Player I chooses an element $u$ of the linear space $\mathcal{B}$ and Player II
 (who does not see $u$) must approximate Player I's  choice based on
 seeing the finite number of linear measurements $\Phi(u)$ of $u$.
The  function $(u,v) \mapsto \frac{\|u-v(\Phi(u))\|}{\|u\|}$
 has no saddle points, so to identify a minmax  solution as a saddle point
  one can  proceed, as in 
  Wald's decision theory
   \cite{Wald:1945}, evidently influenced by von Neumann's  theory of games \cite{VNeumann28}, by
  introducing mixed/randomized strategies and lift the problem to probability measures  over all possible choices for players I and II. 
For the lifted version of the game, the optimal mixed strategy of Player I
is a cylinder measure defined by the norm $\|\cdot\|$ and the  optimal strategy of Player II is a pure strategy because $\|\cdot\|$ is convex. Furthermore if the norm $\|\cdot\|$ is quadratic, then the 
 optimal strategy of Player I is a centered Gaussian field whose covariance operator $Q\,:\mathcal{B}^*\to \mathcal{B}$ is defined by the norm $\|\cdot\|$ and the identity $\|\phi\|_*^2=[\phi, Q \phi]$.
For further references on  Gaussian measures on infinite-dimensional spaces, we refer to
Bogachev \cite{bogachev1998gaussian} and Maniglia and Rhandi \cite{maniglia2004gaussian} (for Hilbert spaces).
See also Janson\index{Janson}  \cite{janson1997gaussian} for  Gaussian fields on Hilbert spaces.
The application of optimal recovery, initially focused on solving linear PDEs \cite{Harder:1972, Duchon:1977, Owhadi:2014}, has been extended to nonlinear PDEs in \cite{chen2021solving} and to general computational graph completion problems in \cite{owhadi2022computational}.

\section{Proof of Theorem~\ref{prop:conditional-mode-optimization}}
\label{appendix:Proof-of-conditional-mode-optimization}

The main ideas required for the proof of 
Theorem~\ref{prop:conditional-mode-optimization} are contained
in Proposition \ref{prop:main-conditional-map}. The proposition
and theorem themselves rest on several lemmas which we collect
together in a preliminary subsection.

First we recall three technical results, concerning small
ball probabilities, from \cite{dashti2013map}.

\begin{lemma}[{\cite[Lem.~3.6]{dashti2013map}}]\label{lem:small-ball-probability-bound}
  Let $\mu = N(0, \mK)$, $r >0$ and $u \in \mX$. Then there exists a constant $\alpha > 0$ indepenent of $u, r$ so that
  \begin{equation*}
    \frac{\mu(B_r(u))}{\mu(B_r(0))} \le  \exp \left( \frac{\alpha}{2} r^2 \right) \exp \left( - \frac{\alpha}{2}\left( \| u\|_\mX - r\right)^2 \right).
  \end{equation*}
\end{lemma}

\begin{lemma}[{\cite[Lem.~3.7]{dashti2013map}}]\label{lem:ON-of-points-outside-CM}
  Suppose $u_0 \not\in \mH(\mu)$, $\{ u_r \}_{r \ge 0} \subset \mX$ and $u_r$ converges
  weakly to $u_0$ in $\mX$ as $r \to 0$. Then for any $\eps >0$ there exists $r >0$ small
  enough so that 
  \begin{equation*}
 \frac{\mu(B_r(u_r))}{\mu(B_r(0))} < \eps \,.
  \end{equation*}
\end{lemma}

\begin{lemma}[{\cite[Lem.~3.9]{dashti2013map}}]\label{lem:weak-but-not-strong-conv-of-maps}
  Consider a sequence $\{u_r\}_{r \ge 0} \subset \mX$ and suppose $u_r$ converges weakly and not
  strongly to $0$ in $\mX$ as $r \to 0$. Then for any $\eps > 0$, there exists $r$ small enough
  such that
  \begin{equation*}
    \frac{\mu(B_r(u_r))}{ \mu(B_r(0))} < \eps.
  \end{equation*}
\end{lemma}
A fourth useful lemma concerning small ball probabilities is:
\begin{lemma}[{\cite[Lem.~4.7.1]{bogachev1998gaussian}}]\label{lem:bogachev-small-ball-lower-bound}
  For all $u \in \mH(\mu)$ it holds that 
  \begin{equation*}
      1 \le \frac{1}{\mu(B_r(0))} \int_{B_r(0)} \exp \left( \langle u, x \rangle_{\mH(\mu)} \right) \dd \mu(x).
  \end{equation*}
\end{lemma}
For our final lemma we recall the following classic result 
(see for example \cite[Cor.~4.7.8]{bogachev1998gaussian})
which is integral to the analysis in the following subsection.

\begin{lemma}\label{lem:OM-lemma}
  Let $\mu = N(0, \mK) \in \PP(\mX)$. Then
  \begin{equation*}
    \lim_{r \to 0} \frac{\mu(B_r(u_1))}{\mu(B_r(u_2))} = \exp \left( \frac{1}{2} \| u_2 \|_{\mH(\mu)}^2
    - \frac{1}{2} \| u_1 \|_{\mH(\mu)}^2 \right), \qquad  \forall u_1, u_2 \in \mH(\mu).
  \end{equation*}
\end{lemma}

Now recall \Cref{def:conditional-mode} of the  conditional mode. Our
goal is to show that such a point is equivalent to a minimizer of
\eqref{conditional-map-optimization-problem}. We start by establishing 
the existence of such minimizers.

\begin{proposition}
  Let $\mu = N(0, \mK)$ and fix $y \in T(\mH(\mu))$ for a continuous map $T: \mX \to \mY$. Then there exists a minimizer $u^y$
  of \eqref{conditional-map-optimization-problem}. 
\end{proposition}

\begin{proof}
  Since $y \in T(\mH(\mu))$ by assumption, then the feasible set $T^{-1}(y) \cap \mH(\mu)$ is non-empty.
   Define $I := \inf \{ \| u\|_{\mH(\mu)} : u \in T^{-1}(y) \}$
  and let $\{u_n\} \in T^{-1}(y)$ be a minimizing sequence. Then for any $\delta >0$ there exists $N=N(\delta)$ so that
  \begin{equation*}
   0 \le I \le \| u_n \|_{\mH(\mu)} \le I + \delta, \qquad \forall n \ge N. 
 \end{equation*}
 Since $\mH(\mu)$ is a Hilbert space and $\{u_n\}$ is bounded we infer the
 existence of a limit point $u^y \in \mH(\mu) $ (possibly along a subsequence) so that
 $u_n$ converges to $u^y$ weakly in $\mH(\mu)$. The weak lower semicontinuity of
 the $\mH(\mu)$-norm now yields, $I \le \| u^y \|_{\mH(\mu)} \le I + \delta$ and
 the result follows since $\delta$ is arbitrary.
\end{proof}


\begin{proposition}\label{prop:main-conditional-map}
Consider $\mu = N(0, \mK)$, a continuous map $T: \mX \to \mY$
  and a point $ y \in T(\mX)$.
 Define
  \begin{equation}\label{def:u-r-sequence}
  u_r : = \argmax_{u \in T^{-1}(y)} \mu(B_r(u)).
\end{equation}
  Then:
  \begin{enumerate}[label=\it (\roman*)]
  \item the maximizer $u_r \in \mX$ exists for every $r>0;$
  \item if $y$ belongs to $T(\mH(\mu))$ then there exists $u^y \in \mH(\mu) \cap T^{-1}(y)$
    and a subsequence of $\{ u_r\}_{r \ge 0}$ which converges to $u^y$ 
    strongly in $\mX$ as $r \to 0;$
  \item if $ y $ belongs to $ T(\mH(\mu)) \cap \supp T_\sharp \mu$, 
and the intersection is not empty, then
  the limit $u^y$ is both a conditional mode of $\mu( \dd u | T(u) = y)$
    and a minimizer of \eqref{conditional-map-optimization-problem}.
\end{enumerate}

\end{proposition}


\begin{proof}
  (i) First observe that by assumption $T^{-1}(y)$ is not empty. 
  By \Cref{lem:small-ball-probability-bound} we deduce that any maximizing
  sequence is bounded in $\mX.$ Extract a weakly convergent subsequence $\{u_r^{(n)}\}_{n \in \mathbb{N}}$ with limit $u_r$. Since $\mX$ is a Hilbert space
  and $T^{-1}(y)$ is closed 
  we conclude that $u_r \in T^{-1}(y).$ 
  The Gaussian measures $\mu(\cdot + u_r^{(n)})$ then converge weakly as $n \to \infty$ to Gaussian measures $\mu(\cdot + u_r)$ \cite{bogachev1998gaussian}. 
  Thus $\mu(B_r(u_r^{(n)})) \to \mu(B_r(u_r))$ since the indicator function 
  of a ball is a bounded measurable function. Hence, since the
  subsequence is a maximizing subsequence, the result is proved.

(ii) Now consider the sequence 
$u_r = \argmax_{u \in T^{-1}(y)} \mu(B_r(u))$, 
indexed over $r \ge 0.$ Our first task is to show that
  $\{ u_r\}_{r \ge 0}$ is bounded in $\mX$. 
By the hypothesis that $y \in T(\mH(\mu))$ we can pick a point 
$u^\star \in \mH(\mu) \cap T^{-1}(y)$, which we will fix for
the remainder of the proof of (ii). Since $u_r$ is, by definition, the maximizer of $\mu(B_r(u))$ 
over $T^{-1}(y)$ then we have that
\begin{equation}\label{trivial-lower-bound}
    \frac{\mu\bigl(B_r(u_r)\bigr)}{\mu\bigl(B_r(u^\star)\bigr)} \ge 1.
\end{equation}
By the Cameron-Martin formula we can further write 
\begin{align*}
  1 & \le \frac{\mu\bigl(B_r(u_r)\bigr)}{\mu\bigl(B_r(u^\star)\bigr)} = 
\frac{\mu\bigl(B_r(u_r)\bigr)}{\mu\bigl(B_r(0)\bigr)} \frac{\mu\bigl(B_r(0)\bigr)}{\mu\bigl(B_r(u^\star)\bigr)}\\
   & = \frac{\mu\bigl(B_r(u_r)\bigr)}{\mu\bigl(B_r(0)\bigr)} \exp\left( \frac{1}{2} \| u^\star \|_{\mH(\mu)}^2 \right) 
   \frac{\mu(B_r(0))}{\int_{B_r(0)} \exp( -\langle u^\star , x \rangle_{\mH(\mu)} ) \dd \mu(x)  }.
  \end{align*}
  An application of \Cref{lem:bogachev-small-ball-lower-bound} yields  the lower bound
  \begin{equation}\label{non-trivial-lower-bound}
    \frac{\mu(B_r(u_r))}{\mu(B_r(0))} \ge \exp\left( - \frac12 \| u^\star\|_{\mH(\mu)}^2 \right).
\end{equation}
  Now suppose, to obtain a contradiction, that $\{ u_r \}_{r \ge 0}$ is not bounded in $\mX$, so that
  for any $R >0$ there exists $r_R$ so that $\| u_{r_R} \|_\mX > R$ with $r_R \to 0$ and $ R \to \infty$.
  Then the lower bound \eqref{non-trivial-lower-bound} 
   contradicts \Cref{lem:small-ball-probability-bound} for
  large $R$ and sufficiently small $r_R$ leading
  to the conclusion that $\{ u_r \}_{r \ge 0}$ is bounded. Since $\mX$ is a Hilbert space
  and $T^{-1}(y)$ is closed 
  we infer
  there exists a point $u^y \in T^{-1}(y)$ and a subsequence $\{ u_r\}$ which converges weakly to $u^y$ in $\mX$ as $r \to 0$. 

Now suppose, again for contradiction, that either: (a) there is no strongly convergent subsequence of $\{ u_r\}$ in $\mX$; or (b)
if there is such a subsequence its limit $u_0$ does not belong to $\mH(\mu)$.
We start with the case (b).
Consider \eqref{non-trivial-lower-bound} and apply \Cref{lem:ON-of-points-outside-CM} with
$\epsilon =  \frac12 \exp( - \frac12 \| u^\star\|_{\mH(\mu)}^2 )$
 to obtain
\begin{equation}\label{contradiction-bound}
 \exp( - \frac12 \| u^\star\|_{\mH(\mu)}^2 )
  \le \frac{\mu(B_r(u_r))}{\mu(B_r(0))} < \frac{1}{2} \exp( - \frac12 \| u^\star\|_{\mH(\mu)}^2 ),
\end{equation}
which is a contradiction and so the limit point $u_0 \in \mH(\mu)$. 
Now consider case (a) where there exists no strongly
convergent subsequence that converges to $u_0$. Then the (sub)sequence $u_r - u_0$ satisfies
the conditions of \Cref{lem:weak-but-not-strong-conv-of-maps}. We can 
then repeat the above argument with the same choice of $\eps$
to obtain \eqref{contradiction-bound} once again which is a contradiction.
This concludes the proof of part (ii).

(iii) 
In what follows we let $u_0 \in T^{-1}(y) \cap \mH(\mu)$ denote the limit of the relabelled subsequence
$\{ u_s \}$ of $\{u_r\}$ as in part (ii). Now suppose either $\{u_s\}$ 
is not bounded in $\mH(\mu)$ or
if it is, it only converges weakly to $u_0$ and not strongly in $\mH(\mu)$.
This implies that $\| u_0 \|_{\mH(\mu)} \le \lim \inf_{s \to 0} \| u_s \|_{\mH(\mu)}$
which in turn implies the existence of a sufficiently small $s$ for which 
$\| u_0 \|_{\mH(\mu)} \le \| u_s \|_{\mH(\mu)}$. Therefore \Cref{lem:OM-lemma} implies that 
$  \limsup_{s \to 0} \frac{\mu(B_s(u_s))}{\mu(B_s(u_0))} \le 1.$
On the other hand, by the definition of $u_s$ we have that $\mu(B_s(u_s)) \ge \mu(B_s(u_0))$ and
so $  \liminf_{s \to 0} \frac{\mu(B_s(u_s))}{\mu(B_s(u_0))} \ge 1,$
from which we conclude that
\begin{equation}\label{interim-display-limit-of-balls}
  \lim_{s \to 0} \frac{\mu(B_s(u_s))}{\mu(B_s(u_0))} = 1.
\end{equation}
By \Cref{def:conditional-mode} it follows that $u_0$ is a
conditional mode. It remains to consider
the setting where $\{ u_s \}$ converges strongly to $u_0$ in $\mH(\mu)$.
Then by the Cameron-Martin formula we have
\begin{equation*}
  \frac{\mu(B_s(u_s))}{\mu(B_s(u_0))}
  = \exp \left( \frac{1}{2} \| u_0 \|_{\mH(\mu)}^2 - \frac{1}{2} \| u_s \|_{\mH(\mu)}^2  \right)
  \frac{\int_{B_s(0)} \exp \left( -\langle u_s, v \rangle_{\mH(\mu)} \right) \mu(\dd v)  }
  {\int_{B_s(0)} \exp \left( -\langle u_0, v \rangle_{\mH(\mu)} \right) \mu(\dd v)}.
\end{equation*}
It follows, from \cite[Lem.~4.7.1; see also proof of Lem.~4.7.2]{bogachev1998gaussian},
that the maps
\begin{equation*}
 u \mapsto \mu(B_s(0))^{-1} \int_{B_s(0)} \exp \left( -\langle u, v \rangle_{\mH(\mu)} \right) \mu(\dd v),
\end{equation*}
are locally Lipschitz on $\mH(\mu)$ from which we infer \eqref{interim-display-limit-of-balls} once again.

We now show that $u_0$ solves \eqref{conditional-map-optimization-problem}. Suppose otherwise,
so that $\| u_0 \|_{\mH(\mu)} - \| u^y \|_{\mH(\mu)} > 0$. By \Cref{lem:OM-lemma} we
have that
\begin{equation*}
  \frac{\mu(B_s(u_0))}{\mu(B_s(u^y))} \le K(s) \exp \left( \frac{1}{2} \| u^y \|_{\mH(\mu)}^2 -
  \frac{1}{2} \| u_0 \|_{\mH(\mu)}^2 \right),
\end{equation*}
with $K(s) \to 1$ as $s \to 0$. Now choose $\tilde{s}$ sufficiently small so that
\begin{equation*}
1 \le K(s) < \exp \left( \frac{1}{2} \| u_0 \|_{\mH(\mu)}^2 -
  \frac{1}{2} \| u^y \|_{\mH(\mu)}^2 \right)
\end{equation*}
for any $s < \tilde{s}$. Then by the above display we
have
\begin{equation*}
  \frac{\mu(B_s(u_0))}{\mu(B_s(u^y))} <1.
\end{equation*}
Using this bound and \eqref{interim-display-limit-of-balls} we can then write
\begin{equation*}
  \limsup_{s \to 0} \frac{\mu(B_s(u_s))}{\mu(B_s(u^y))}
  = \limsup_{s \to 0} \frac{\mu(B_s(u_s))}{\mu(B_s(u_0))} \frac{\mu(B_s(u_0))}{\mu(B_s(u^y))}
  < \limsup_{s\to 0} \frac{\mu(B_s(u_s))}{\mu(B_s(u_0))} \le 1, 
\end{equation*}
which is a contradiction since by the definition of $u_s$ we have $\mu(B_s(u_s)) \ge \mu(B_s(u^y))$
for any $s >0$. Thus $u_0$ solves \eqref{conditional-map-optimization-problem}.
{}
\end{proof}

\begin{proof}[Proof of Theorem~\ref{prop:conditional-mode-optimization}]
  First let $u^y$ be a conditional mode and take the sequence $\{u_r\}_{r \ge 0}$ as in \eqref{def:u-r-sequence}.
  By \Cref{prop:main-conditional-map} there exists a relabelled subsequence $\{u_s\}$
  which converges strongly in $\mX$ to $u_0 \in \mH(\mu) \cap T^{-1}(y)$ and 
  $u_0$ is
  also a conditional mode and so by \Cref{def:conditional-mode}
 it holds that
$    \lim_{r \to 0} \frac{\mu(B_r(u_r))}{\mu(B_r(u_0))} = 1. $
  Since $u^y$ is also a conditional mode we have
  \begin{equation*}
    \lim_{r \to 0} \frac{\mu(B_r(u^y))}{\mu(B_r(u_0))} =
    \lim_{r \to 0} \frac{\mu(B_r(u^y))}{\mu(B_r(u_r))}
    \lim_{r \to 0} \frac{\mu(B_r(u_r))}{\mu(B_r(u_0))} =1. 
  \end{equation*}
  We infer from \Cref{lem:ON-of-points-outside-CM} that $u^y \in \mH(\mu) \cap T^{-1}(y)$
  since otherwise the limit $\lim_{r \to 0} \frac{\mu(B_r(u^y))}{\mu(B_r(u_r))}$ would vanish.
  Now suppose $u^y$ does not solve \eqref{conditional-map-optimization-problem}. We can
  obtain a contradiction by repeating the last step of the proof of \Cref{prop:main-conditional-map}.

  To prove the converse statement let $u^y$ be a solution of
  \eqref{conditional-map-optimization-problem} with $u_0$ defined as before. Then \Cref{lem:OM-lemma} implies  $\lim_{r \to 0} \frac{\mu(B_r(u_0))}{\mu(B_r(u^y))}=1$,
  and so we have
  \begin{equation*}
    \lim_{r \to 0} \frac{\mu(B_r(u_r))}{\mu(B_r(u^y))}
    = \lim_{r \to 0} \frac{\mu(B_r(u_r))}{\mu(B_r(u_0))}
     \lim_{r \to 0} \frac{\mu(B_r(u_0))}{\mu(B_r(u^y))} = 1.
   \end{equation*}
The result follows from \Cref{def:conditional-mode}.
\end{proof}

\end{appendices}


\bibliography{KF_PDE_references.bib}

\end{document}